\documentclass[twoside,11pt]{article}

\usepackage{blindtext}

\usepackage[abbrvbib, preprint]{jmlr2e}

\usepackage{amsmath}
\usepackage{booktabs}
\usepackage{multirow}
\usepackage{subcaption}

\newtheorem{assumption}{Assumption}

\usepackage{lastpage}
\jmlrheading{XX}{2026}{1-\pageref{LastPage}}{Submitted 3/26}{Under review}{XX-0000}{Yi Wei, Xuan Qi, and Furao Shen}

\ShortHeadings{Region Seeding via Pre-Activation Regularization}{Wei, Qi, Shen}
\firstpageno{1}

\begin{document}

\title{Region Seeding via Pre-Activation Regularization: A Geometric View of Piecewise Affine Neural Networks}

\author{\name Yi Wei$^{*}$ \email ywei@smail.nju.edu.cn \\
       \addr State Key Laboratory of Novel Software Technology\\
       School of Intelligence Science and Technology\\
       Nanjing University, Jiangsu, China\\
       \AND
       \name Xuan Qi$^{*}$ \email xuan.qi@iit.it \\
       \addr AI for Good\\
       Istituto Italiano di Tecnologia, Genoa, Italy\\
       DITEN\\
       University of Genoa, Genoa, Italy\\
       \AND
       \name Furao Shen$^{\dagger}$ \email frshen@nju.edu.cn \\
       \addr State Key Laboratory of Novel Software Technology\\
       School of Artificial Intelligence\\
       Nanjing University, Jiangsu, China}
\editor{My editor}

\maketitle
\begingroup
\renewcommand{\thefootnote}{\fnsymbol{footnote}}
\footnotetext[1]{Equal contribution. \quad $^{\dagger}$Corresponding author.}
\endgroup

\begin{abstract}
Deep networks with continuous piecewise affine activations induce polyhedral partitions of the input space, making the number of realized affine regions a natural measure of expressive capacity and a key determinant of how well the model can approximate nonlinear target functions. In practice, standard training realizes far fewer region refinements in data-visited neighborhoods than the architecture could in principle support, while existing region-count theory is primarily architectural and offers little guidance on how optimization shapes the realized partition near the data. Our theory provides a sufficient condition under which bringing neuron switching surfaces sufficiently close to data points ensures their intersection with local neighborhoods, which in turn implies a strict increase in the local affine-region count, yielding a principled training-time handle for seeding data-relevant partitions early in optimization. Guided by these results, we propose a plug-and-play region-seeding regularizer that encourages early partitioning while allowing task-driven refinement to dominate later in training. Experiments show that the regularizer increases the number of realized affine regions via exact enumeration and improves overall performance on toy datasets, while also improving early-stage accuracy and achieving comparable (or slightly improved) final accuracy on ImageNet-1k for classical models.
\end{abstract}

\begin{keywords}
Deep Learning, Continuous Piecewise Affine, Expressivity, Regularization
\end{keywords}

\section{Introduction}
Deep neural networks learn hierarchical representations by composing multiple layers of simple transformations~\cite{book36}. Deep networks with continuous piecewise affine (CPA) activations, such as ReLU~\cite{book5} and LeakyReLU~\cite{book6}, implement mappings that are affine on each cell of a finite polyhedral partition of the input space~\cite{book1}. The number of such affine regions is therefore a natural measure of a network's expressive capacity~\cite{book2}. This viewpoint is particularly natural for classification: with a linear readout on top of the final representation, each pairwise logit difference is itself CPA, so decision boundaries can be expressed as unions of polyhedral facets aligned with the induced partition. From a learning perspective, increasing the number of regions in the parts of the input space reached by the data expands the set of local affine behaviors the model can deploy during optimization, enabling finer-grained decision surfaces and more flexible function approximation~\cite{book3}.

A substantial body of theory bounds the total number of affine regions a CPA network can realize as a function of depth and width. Yet, in practice, trained models often exhibit substantially less region refinement in the parts of the input space visited by the data, and region counts are not straightforward to monitor or optimize directly during large-scale training~\cite{book4}. Moreover, much of the existing theory is architectural and global: it characterizes achievable region complexity, but offers limited guidance on how standard optimization procedures shape the realized partition near the data. Motivated by this, we investigate an explicit training-time mechanism that biases region formation toward data-relevant neighborhoods, while keeping the underlying architecture unchanged.

Our key observation is geometric and leads to a trainable surrogate. Each hidden unit induces a switching surface that separates its two affine regimes. If many such surfaces intersect a small neighborhood around a data point, that neighborhood is more likely to be subdivided into additional affine regions, increasing local model flexibility. Pre-activation magnitudes provide a direct handle on this effect: small magnitudes indicate that the corresponding switching surface lies close to the data point, making intersection with a nearby neighborhood more likely.

We formalize this intuition with two local theorems that connect an observable training quantity---pre-activation magnitude---to the emergence of affine regions near the data. Theorem~\ref{thm:monotone-cells} shows that, under mild regularity, whenever additional neuron zero sets intersect a small convex neighborhood around a data point, the CPA partition restricted to that neighborhood necessarily gains more affine regions. Theorem~\ref{thm:distance-intersection} provides a complementary guarantee that is directly usable in training: it gives a pre-activation-based sufficient condition ensuring that a neuron's zero set intersects a local neighborhood around the data point. Together, these results justify a simple training principle: encourage pre-activations to be small in magnitude on data early in optimization to seed more local partitions, and then gradually relax this bias to allow task-driven refinement.

Motivated by these geometric and local-region results, we introduce a plug-and-play region-seeding regularizer that biases optimization toward forming additional affine regions in neighborhoods of the data during the early stages of training. Concretely, the regularizer penalizes the (normalized) Euclidean magnitude of the pre-activation tensors entering each activation module on a mini-batch. We aggregate these penalties across activation modules using nonnegative, depth-dependent weights that control which parts of the network are emphasized, and apply an epoch-dependent annealing factor so the regularization effect is strongest early in training and gradually reduced thereafter. While our theoretical motivation is developed for CPA networks, we also apply the same penalty to non-CPA activations (e.g., GELU~\cite{book7}) as a practical heuristic.

Our contributions are as follows.
\begin{itemize}
\item We provide a local geometric foundation for region seeding: Theorem~\ref{thm:monotone-cells} shows that each additional neuron zero set intersecting a data neighborhood strictly increases the local affine-region count, and Theorem~\ref{thm:distance-intersection} gives a pre-activation-based sufficient condition guaranteeing such intersections, linking pre-activation control to provable local partition growth.
\item We propose a depth-weighted, annealed region-seeding regularizer that implements these insights using only batchwise pre-activations, offering an architecture-preserving training-time mechanism to influence where and when local affine regions form.
\item We validate our approach via exact affine-region enumeration for standard CPA networks on toy datasets, showing that the regularizer increases the number of realized regions and improves overall performance relative to standard training, and that it similarly improves early-stage accuracy while achieving comparable (or slightly improved) final accuracy on ImageNet-1k for classical models.
\end{itemize}

\section{Related Work}
\label{sec:related}

\paragraph{Affine-region geometry.}
Prior work studies CPA expressivity via the polyhedral partitions they induce, including algorithms for computing or extracting regions in specific architectures~\cite{book8,book9,book10,book17} and architecture-level region-count analyses (e.g., GNNs, convolutional ReLU models, and tropical-geometry viewpoints)~\cite{book11,book12,book13,book19,book29,book40}. These results quantify attainable complexity but are largely silent on how optimization concentrates refinement near the data.

\paragraph{Local complexity and training dynamics.}
Several works consider data-dependent notions of complexity and relate them to representation learning and optimization~\cite{book14}, or empirically/geom\-etrically analyze the evolution of partitions and decision regions during training~\cite{book3,book15,book16,book39,book41}. Collectively, these works emphasize that the realized partitions are shaped by the interaction between the data distribution and optimization dynamics.

\paragraph{Expressive capacity.}
Depth/width scaling laws, training-dependent effects, and universal-approximation results provide global representational guarantees~\cite{book18,book20,book21,book22,book23,book24,book25,book26,book27,book28,book30}, but typically do not yield actionable conditions for controlling the \emph{realized} partition near the data.

\paragraph{Advantage of our approach.}
We provide a local geometric link from an observable forward-pass quantity (pre-activation magnitude) to a sufficient condition for neuron switching surfaces to intersect data neighborhoods, which under mild regularity implies strict local region growth.

\section{Methodology}

\paragraph{Notation.}
We denote the input dimension by $d$ and write $x\in\mathbb{R}^d$ for a data point. The network has $L$ layers, and layer $\ell\!-\!1$ has width $m_{\ell-1}$ with feature map $h_{\ell-1}(x)\in\mathbb{R}^{m_{\ell-1}}$. For neuron $i$ in layer $\ell$, $w_{\ell,i}\in\mathbb{R}^{m_{\ell-1}}$ and $b_{\ell,i}\in\mathbb{R}$ are its weight and bias, with pre-activation $z_{\ell,i}(x)$ and input-space zero set $Z_{\ell,i}:=\{x\in\mathbb{R}^d:\,z_{\ell,i}(x)=0\}$. Let $\theta$ denote the collection of all network parameters, and let $\mathcal{P}_\theta$ be the induced CPA partition of the input space; $R_{\mathrm{parent}}(x)\in\mathcal{P}_\theta$ is the unique cell containing $x$. On $R_{\mathrm{parent}}(x)$, $h_{\ell-1}$ is affine with linear part $A_{\ell-1}\in\mathbb{R}^{m_{\ell-1}\times d}$ and offset $q_{\ell-1}\in\mathbb{R}^{m_{\ell-1}}$, so the local zero set corresponds to an affine hyperplane with normal $a_{\ell,i}\in\mathbb{R}^d$ and offset $c_{\ell,i}\in\mathbb{R}$. For $\varepsilon>0$, $P_\varepsilon(x)\subset\mathbb{R}^d$ denotes a convex closed local polytope (neighborhood) with $x\in\mathrm{int}(P_\varepsilon(x))$.

\subsection{Geometry behind our approach}
Recall that for layer $\ell\in\{1,\dots,L\}$ and neuron $i$, the pre-activation is
\begin{equation}
\label{eq:def-z-again}
z_{\ell,i}(x) \;=\; w_{\ell,i}^\top h_{\ell-1}(x) + b_{\ell,i},
\end{equation}
and the corresponding input-space zero set is
\begin{equation}
\label{eq:def-Z-again}
Z_{\ell,i}\;=\;\bigl\{\,x\in\mathbb{R}^d:\; z_{\ell,i}(x)=0\,\bigr\}.
\end{equation}
Consider a data point $x$ and its parent cell $R_{\mathrm{parent}}(x)$ in the current CPA partition of the input space: $x$ is guaranteed to lie inside this convex region. When we add neurons in the next layer, each neuron $(\ell,i)$ induces a hyperplane in input space, and the intersections of these hyperplanes with $R_{\mathrm{parent}}(x)$ can further subdivide this parent region into more child regions. Intuitively, the closer (in Euclidean distance) the neuron's hyperplane passes to $x$ inside $R_{\mathrm{parent}}(x)$, the more reliably it will intersect this parent polytope instead of missing it, thereby increasing the likelihood that the parent region is split by this neuron.

At the feature level, the pre-activation $z_{\ell,i}(x) = w_{\ell,i}^\top h_{\ell-1}(x) + b_{\ell,i}$ defines a hyperplane in the $h_{\ell-1}$-space,
\begin{equation}
\label{eq:H-hyperplane}
H_{\ell,i}
:=
\bigl\{\,h\in\mathbb{R}^{m_{\ell-1}} : w_{\ell,i}^\top h + b_{\ell,i} = 0\,\bigr\}.
\end{equation}
The Euclidean distance from the feature vector $h_{\ell-1}(x)$ to this hyperplane is
\begin{equation}
\label{eq:feature-hyperplane}
\begin{aligned}
\mathrm{dist}\bigl(h_{\ell-1}(x),H_{\ell,i}\bigr)
&=
\frac{\bigl|w_{\ell,i}^\top h_{\ell-1}(x) + b_{\ell,i}\bigr|}
     {\|w_{\ell,i}\|}
\\[4pt]
&=
\frac{|z_{\ell,i}(x)|}{\|w_{\ell,i}\|}.
\end{aligned}
\end{equation}

Thus, in feature space, a small pre-activation magnitude $|z_{\ell,i}(x)|$ means that the hyperplane $H_{\ell,i}$ passes close to the feature $h_{\ell-1}(x)$.

For the CPA partition in the \emph{input} space, we care about the distance from $x$ to the input-space zero set $Z_{\ell,i}$ \emph{inside its parent cell}. Inside the affine region $R_{\mathrm{parent}}(x)$ of $h_{\ell-1}$ that contains $x$, the map $h_{\ell-1}$ coincides with an affine map. More precisely, there exist a matrix $A_{\ell-1}\in\mathbb{R}^{m_{\ell-1}\times d}$ and a vector $q_{\ell-1}\in\mathbb{R}^{m_{\ell-1}}$, determined by the activation pattern of all previous layers on $R_{\mathrm{parent}}(x)$, such that
\begin{equation}
\label{eq:h-affine-parent}
h_{\ell-1}(x)
\;=\;
A_{\ell-1}\,x + q_{\ell-1},
\qquad \text{for all } x\in R_{\mathrm{parent}}(x).
\end{equation}
Substituting \eqref{eq:h-affine-parent} into \eqref{eq:def-z-again} shows that $z_{\ell,i}$ is affine in $x$ on $R_{\mathrm{parent}}(x)$ and we can write
\begin{equation}
\label{eq:local-affine}
z_{\ell,i}(x)
\;=\;
a_{\ell,i}^\top x + c_{\ell,i},
\qquad \text{for all } x\in R_{\mathrm{parent}}(x),
\end{equation}
where we identify the input-space normal vector and offset as
\begin{equation}
\label{eq:a-c-definition}
a_{\ell,i} \;:=\; A_{\ell-1}^\top w_{\ell,i},
\qquad
c_{\ell,i} \;:=\; w_{\ell,i}^\top q_{\ell-1} + b_{\ell,i}.
\end{equation}
Thus $Z_{\ell,i}\cap R_{\mathrm{parent}}(x)$ is a $(d\!-\!1)$-dimensional affine
hyperplane patch in input space with normal $a_{\ell,i}$. The distance from $x$
to the corresponding \emph{local hyperplane}
\(
H_{\ell,i}:=\{y\in\mathbb{R}^d:\ a_{\ell,i}^\top y+c_{\ell,i}=0\}
\)
is
\begin{equation}
\label{eq:dist}
\mathrm{dist}\bigl(x,H_{\ell,i}\bigr)
=
\frac{\bigl|a_{\ell,i}^\top x + c_{\ell,i}\bigr|}{\|a_{\ell,i}\|}
=
\frac{|z_{\ell,i}(x)|}{\|a_{\ell,i}\|}.
\end{equation}
When the orthogonal projection of $x$ onto $H_{\ell,i}$ lies inside
$R_{\mathrm{parent}}(x)$ (e.g., for sufficiently small neighborhoods contained
in $R_{\mathrm{parent}}(x)$), this also equals
$\mathrm{dist}\bigl(x, Z_{\ell,i}\cap R_{\mathrm{parent}}(x)\bigr)$.
Consequently, for fixed local geometry (i.e., fixed $\|a_{\ell,i}\|$ on the
parent cell), driving $|z_{\ell,i}(x)|$ toward zero brings the neuron's switching
surface closer to $x$ in input space. Eq.\,\eqref{eq:dist} is the input-space analogue of the point-to-hyperplane distance formula, with input-space normal $a_{\ell,i}$.
In practice, we directly observe and control $z_{\ell,i}(x)=w_{\ell,i}^\top h_{\ell-1}(x)+b_{\ell,i}$ during training, while the normal length $\|a_{\ell,i}\|$ is determined by the accumulated linear mappings of the previous layers.
Therefore, encouraging small $|z_{\ell,i}(x)|$ on data is an effective way to reduce the distance from $x$ to the neuron's switching surface inside its parent region, increasing the chance that it intersects the corresponding parent polytope. Figure~\ref{fig:geometry-overview} provides a schematic illustration of the geometric intuition: for an input $x$, smaller $\mathrm{dist}\!\left(x,\text{neurons}\right)$ increases the likelihood that these surfaces intersect a local neighborhood around $x$, thereby promoting finer local partitioning of the input space.
\begin{figure}[t]
  \centering
  \includegraphics[width=0.98\linewidth]{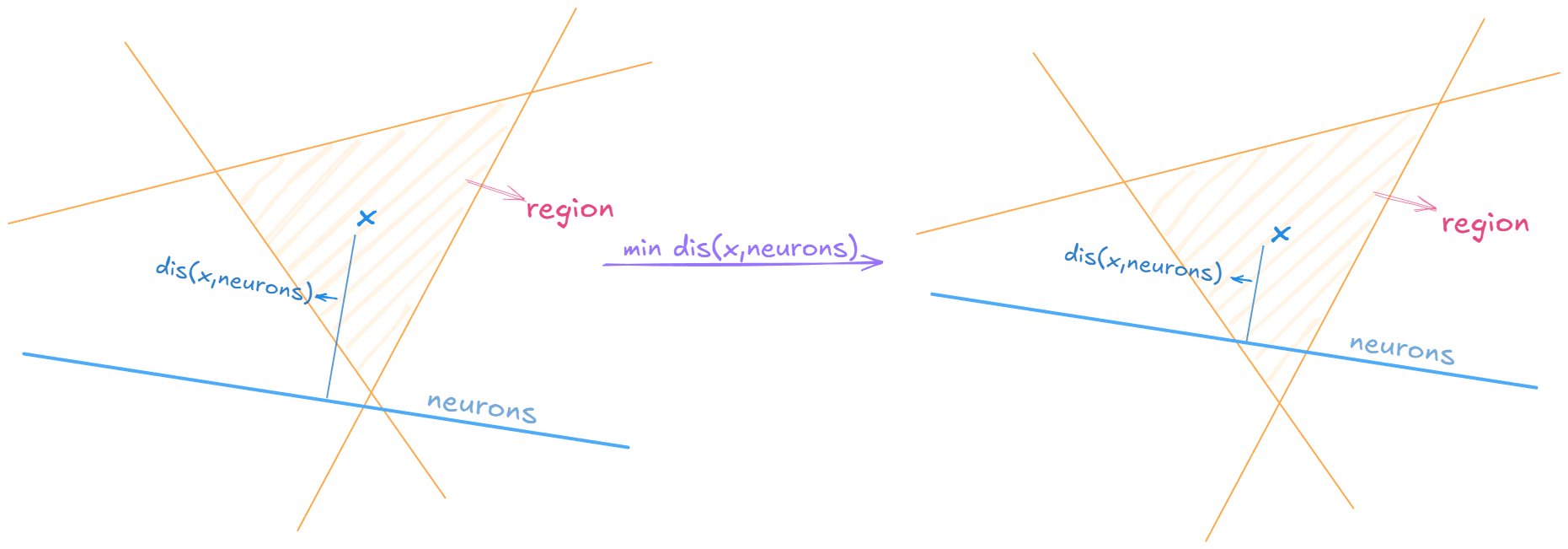}
  \caption{Schematic illustration of the region-seeding intuition. For an input $x$, decreasing the distances $\mathrm{dist}\!\left(x,\text{neurons}\right)$ makes it more likely that neuron-induced switching surfaces intersect a local neighborhood around $x$, which in turn promotes finer local partitioning near the data.}
  \label{fig:geometry-overview}
\end{figure}

\subsection{Local intersection and region-growth theory}

\begin{theorem}[Local region growth from additional zero-set intersections]
\label{thm:monotone-cells}
Fix $x\in\mathbb{R}^d$ and $\varepsilon>0$. Let $P_\varepsilon(x)\subset\mathbb{R}^d$ be a convex closed polytope
with $x\in\mathrm{int}\bigl(P_\varepsilon(x)\bigr)$. Let $\mathcal{J}$ be finite and
$\{z_j:\mathbb{R}^d\to\mathbb{R}\}_{j\in\mathcal{J}}$ be continuous piecewise-affine (CPA) functions with
zero sets $Z_j:=\{y:\,z_j(y)=0\}$. Let $\mathcal{P}_{\mathcal{J}}$ denote the induced sign partition, i.e.,
the collection of full-dimensional open connected cells on which every $z_j$ has a constant strict sign.
Assume $\mathcal{J}$ satisfies Assumption~\ref{assump:local-nondegeneracy} inside $P_\varepsilon(x)$.

Define the set of locally intersecting indices and the local region count by
\begin{equation}
\label{eq:def-SI}
\begin{aligned}
\mathcal{S}_\varepsilon(x;\mathcal{J})
&:=\Bigl\{\,j\in\mathcal{J}:\;
Z_j\cap \mathrm{int}\bigl(P_\varepsilon(x)\bigr)\neq\varnothing\,\Bigr\},\\
I_\varepsilon(x;\mathcal{J})
&:=\bigl|\mathcal{S}_\varepsilon(x;\mathcal{J})\bigr|.
\end{aligned}
\end{equation}
\begin{equation}
\label{eq:def-N}
N_\varepsilon(x;\mathcal{J})
:=\#\{R\in\mathcal{P}_{\mathcal{J}}:\;R\cap P_\varepsilon(x)\neq\varnothing\}.
\end{equation}
Then:
\begin{enumerate}
\item[\textbf{(i)}]
Let $j^\star\notin\mathcal{J}$ and $\mathcal{J}':=\mathcal{J}\cup\{j^\star\}$. Suppose there exists a cell
$R\in\mathcal{P}_{\mathcal{J}}$ such that
\begin{equation}
\label{eq:cut-cond-compact}
\begin{aligned}
Z_{j^\star}\cap R\cap \mathrm{int}\bigl(P_\varepsilon(x)\bigr)
&\neq \varnothing,\\
\exists\,y^\pm\in R\cap \mathrm{int}\bigl(P_\varepsilon(x)\bigr)
\ \text{s.t.}\quad
\pm z_{j^\star}(y^\pm)&>0.
\end{aligned}
\end{equation}

Assume additionally that $\mathcal{J}'$ satisfies Assumption~\ref{assump:local-nondegeneracy} inside $P_\varepsilon(x)$.
Then
\begin{equation}
\label{eq:monotone-plus1-compact}
N_\varepsilon(x;\mathcal{J}')\ \ge\ N_\varepsilon(x;\mathcal{J})+1.
\end{equation}

\item[\textbf{(ii)}]
In particular,
\begin{equation}
\label{eq:linear-lb-compact}
N_\varepsilon(x;\mathcal{J})\ \ge\ 1+I_\varepsilon(x;\mathcal{J}).
\end{equation}

\item[\textbf{(iii)}]
If the $I_\varepsilon(x;\mathcal{J})$ intersecting traces
$\{Z_j\cap \mathrm{int}(P_\varepsilon(x)):\ j\in\mathcal{S}_\varepsilon(x;\mathcal{J})\}$
are in local general position inside $P_\varepsilon(x)$ and each such trace is a single $(d\!-\!1)$-dimensional
affine hyperplane patch, then
\begin{equation}
\label{eq:gpb-upper-local-compact}
N_\varepsilon(x;\mathcal{J})\ \le\ \sum_{k=0}^{d}\binom{I_\varepsilon(x;\mathcal{J})}{k}.
\end{equation}
\end{enumerate}
\end{theorem}

Theorem~\ref{thm:monotone-cells} formalizes our first desideratum: increasing the number of neuron zero sets intersecting a convex parent region (and hence any small convex neighborhood of a data point it contains) necessarily increases the number of local affine regions within that region.

\begin{theorem}[Local neuron--region intersection via pre-activation magnitude]
\label{thm:distance-intersection}
Fix $x\in\mathbb{R}^d$ and let $R_{\mathrm{parent}}(x)$ be its parent cell in the
CPA partition $\mathcal{P}_\theta$. Let $P_\varepsilon(x)\subset\mathbb{R}^d$ be a
convex closed polytope such that
\begin{equation}
x\in\mathrm{int}\bigl(P_\varepsilon(x)\bigr),
\qquad
P_\varepsilon(x)\subset R_{\mathrm{parent}}(x).
\end{equation}
For any nonzero vector $v\in\mathbb{R}^d$, define the directional thickness of
$P_\varepsilon(x)$ at $x$ along $v$ by
\begin{equation}
\begin{aligned}
\delta_\varepsilon(x;v)
&:=
\sup\Bigl\{ t>0 : x + s\,\frac{v}{\|v\|} \in \mathrm{int}\bigl(P_\varepsilon(x)\bigr)
\\
&\qquad\qquad\quad
\text{for all } s\in[-t,t]\Bigr\}.
\end{aligned}
\end{equation}
Then $\delta_\varepsilon(x;v)>0$ for all $v\neq 0$.

Consider any neuron $(\ell,i)$ and its pre-activation
\begin{equation}
z_{\ell,i}(y) = w_{\ell,i}^\top h_{\ell-1}(y) + b_{\ell,i}.
\end{equation}
On $R_{\mathrm{parent}}(x)$, $h_{\ell-1}$ coincides with an affine map
$h_{\ell-1}(y)=A_{\ell-1}y+q_{\ell-1}$, so $z_{\ell,i}$ admits the local affine
representation
\begin{equation}
z_{\ell,i}(y)
=
a_{\ell,i}^\top y + c_{\ell,i},
\qquad y\in R_{\mathrm{parent}}(x),
\end{equation}
with
\begin{equation}
a_{\ell,i} := A_{\ell-1}^\top w_{\ell,i},
\qquad
c_{\ell,i} := w_{\ell,i}^\top q_{\ell-1} + b_{\ell,i}.
\end{equation}
Assume $a_{\ell,i}\neq 0$ and define the associated input-space hyperplane
\begin{equation}
H_{\ell,i}
:=
\bigl\{y\in\mathbb{R}^d : a_{\ell,i}^\top y + c_{\ell,i} = 0\bigr\}.
\end{equation}

If the pre-activation at $x$ satisfies
\begin{equation}
\label{eq:intersection-condition-polytope}
\bigl|z_{\ell,i}(x)\bigr|
=
\bigl|a_{\ell,i}^\top x + c_{\ell,i}\bigr|
\;<\;
\delta_\varepsilon\bigl(x;a_{\ell,i}\bigr)\,\bigl\|a_{\ell,i}\bigr\|,
\end{equation}
then the neuron zero set intersects the interior of the local polytope:
\begin{equation}
Z_{\ell,i}\cap\mathrm{int}\bigl(P_\varepsilon(x)\bigr)\neq\varnothing,
\end{equation}
and hence $(\ell,i)\in\mathcal{S}_\varepsilon(x)$.

In particular, for fixed $P_\varepsilon(x)$ and local normal vector $a_{\ell,i}$,
making the pre-activation magnitude $\bigl|z_{\ell,i}(x)\bigr|$ sufficiently small
is a sufficient condition for the neuron $(\ell,i)$ to contribute a zero-set
intersection inside $P_\varepsilon(x)$.
\end{theorem}

Theorem~\ref{thm:distance-intersection} provides the complementary, actionable link needed for training: it turns the geometric notion of ``a neuron's zero set intersects a data neighborhood'' into a condition expressed directly in terms of the neuron's pre-activation at the data point. In particular, it shows that if \(|z_{\ell,i}(x)|\) is sufficiently small relative to the local scale of the neighborhood along the induced normal direction, then the corresponding zero set must cross the interior of that neighborhood. This result motivates our regularization choice: by biasing pre-activations toward zero on data early in optimization, we can systematically increase the number of local neuron--region intersections, which then translates into local region growth via Theorem~\ref{thm:monotone-cells}. Proofs are provided in Appendix~\ref{Proof1} and~\ref{Proof2}.

\subsection{A plug-and-play region-seeding regularizer}

\paragraph{Layerwise pre-activation penalty.}
Given a mini-batch \(B=\{x_b\}_{b=1}^{|B|}\), we apply a penalty to the
\emph{pre-activation tensors} entering each nonlinearity module (i.e., the inputs
to activation modules). For an activation module indexed by \(\ell\), let
\(u_\ell(x_b)\) denote this pre-activation tensor for sample \(x_b\), and let
\(n_\ell\) be the number of elements in \(\mathrm{vec}(u_\ell(x_b))\) (independent
of \(b\)). We define the batchwise layer penalty
\begin{equation}
\label{eq:rlayer}
\mathcal{R}_\ell(\theta;B)
\;:=\;
\frac{1}{|B|}
\sum_{b=1}^{|B|}
\frac{\bigl\|\mathrm{vec}\bigl(u_\ell(x_b)\bigr)\bigr\|_2}{n_\ell}.
\end{equation}
This normalized \(\ell_2\) penalty directly encourages small pre-activation
magnitudes on data, using only forward-pass quantities. 

\paragraph{Layer aggregation and annealing.}
We aggregate \eqref{eq:rlayer} across activation modules with nonnegative weights
\(\lambda_\ell\) and an epoch-dependent annealing factor \(\eta(t)\in[0,1]\):
\begin{equation}
\label{eq:rtotal}
\mathcal{R}(\theta;B,t)
\;=\;
\alpha\,\eta(t)\sum_{\ell} \lambda_\ell\,\mathcal{R}_\ell(\theta;B),
\end{equation}
where \(\alpha>0\) controls the overall strength. The weights \(\lambda_\ell\)
control the relative contribution of each activation module to the regularizer and
can be used to emphasize different depths. The annealing factor \(\eta(t)\) is chosen to decrease over the course of training so that the regularizer primarily shapes early optimization.

\paragraph{Training objective.}
We optimize the standard task loss augmented with the region-seeding regularizer:
\begin{equation}
\label{eq:loss-total}
\mathcal{L}(\theta;B,t)
\;=\;
\mathcal{L}_{\mathrm{task}}(\theta;B) \;+\; \mathcal{R}(\theta;B,t).
\end{equation}

\paragraph{Connection to the Theorems.}
The penalty \eqref{eq:rlayer} encourages pre-activations to be small on data,
which reduces the distance from a data point to the corresponding neuron switching
surface within its local affine region (Theorem~\ref{thm:distance-intersection}),
thereby increasing the likelihood of neuron--neighborhood intersections. By
Theorem~\ref{thm:monotone-cells}, additional such intersections lead to increased
local affine-region counts. In this way, \eqref{eq:rtotal} operationalizes the
theoretical mechanism of region seeding using a lightweight, training-time
regularizer. While our theory is developed for CPA networks, we also apply the
same penalty to non-CPA activations (e.g., GELU) as a practical heuristic.

\section{Experiments}

\subsection{Toy-data experiments}
\label{subsec:toy}

\begin{figure}[ht]
  \centering
  \setlength{\tabcolsep}{6pt}
  \renewcommand{\arraystretch}{0}
  \begin{tabular}{@{}c@{}}
    \begin{subfigure}[t]{0.95\textwidth}
      \centering
      \includegraphics[width=\textwidth]{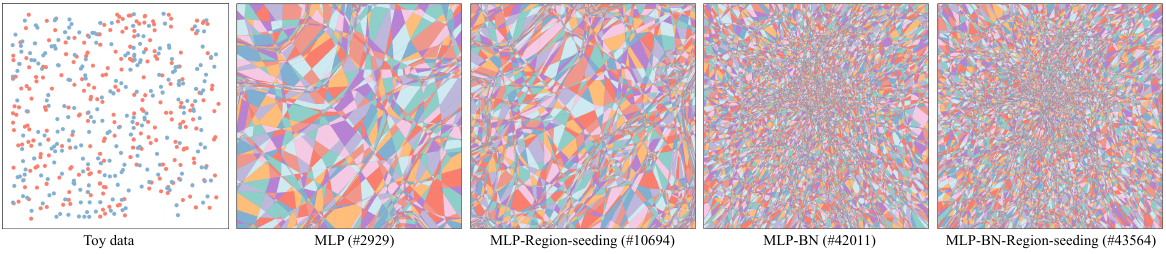}
      \caption{Feedforward MLP}
      \label{fig:toy_regions_vis_mlp}
    \end{subfigure}
    \\[2pt]
    \begin{subfigure}[t]{0.95\textwidth}
      \centering
      \includegraphics[width=\textwidth]{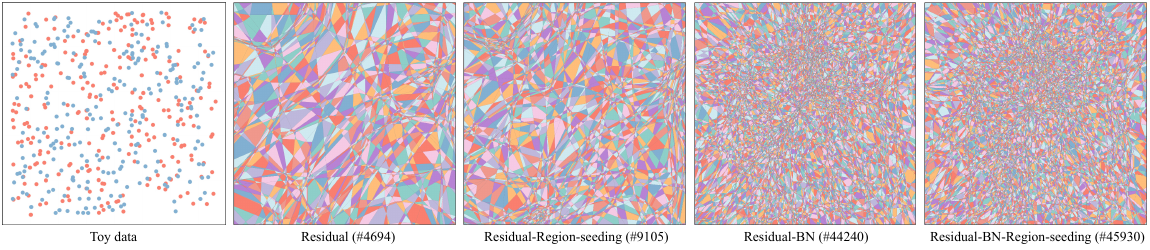}
      \caption{Residual connections}
      \label{fig:toy_regions_vis_residual}
    \end{subfigure}
  \end{tabular}
  \caption{Toy dataset: exact affine-region partition visualizations. Each row is a composite visualization containing five panels (left to right): the toy dataset and four input-space affine partitions obtained by exact enumeration for the corresponding architecture and training variant. The numbers in parentheses denote the realized affine-region count at epoch 1000.}
  \label{fig:toy_regions_vis}
\end{figure}

\begin{figure*}[ht]
  \centering
  \setlength{\tabcolsep}{6pt}
  \renewcommand{\arraystretch}{0}
  \begin{tabular}{@{}cc@{}}
    \begin{subfigure}[t]{0.48\textwidth}
      \centering
      \includegraphics[width=\textwidth]{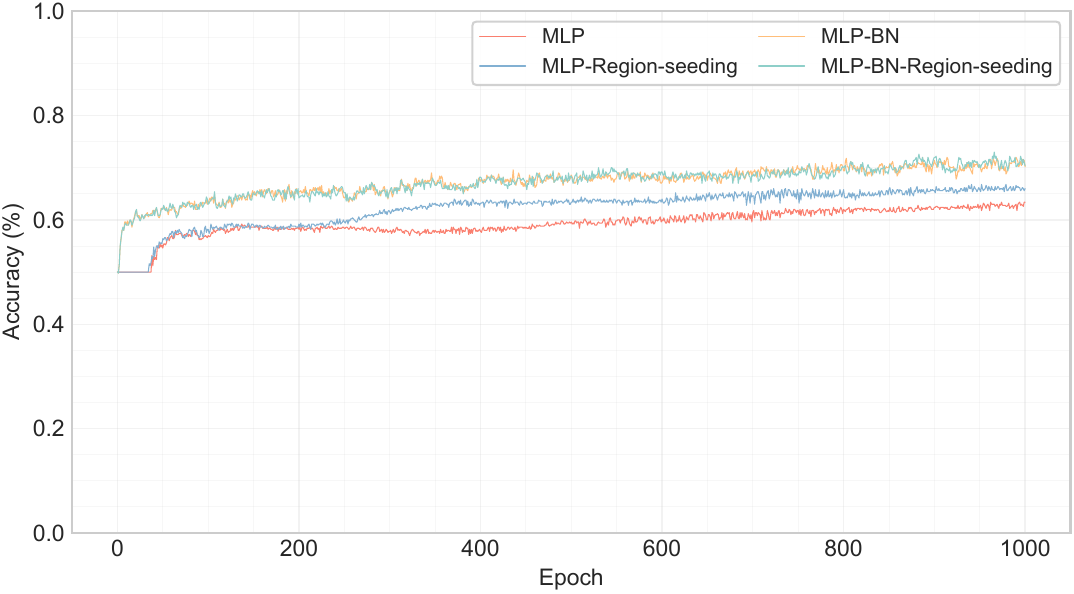}
      \caption{Feedforward MLP}
      \label{fig:toy_acc_mlp}
    \end{subfigure} &
    \begin{subfigure}[t]{0.48\textwidth}
      \centering
      \includegraphics[width=\textwidth]{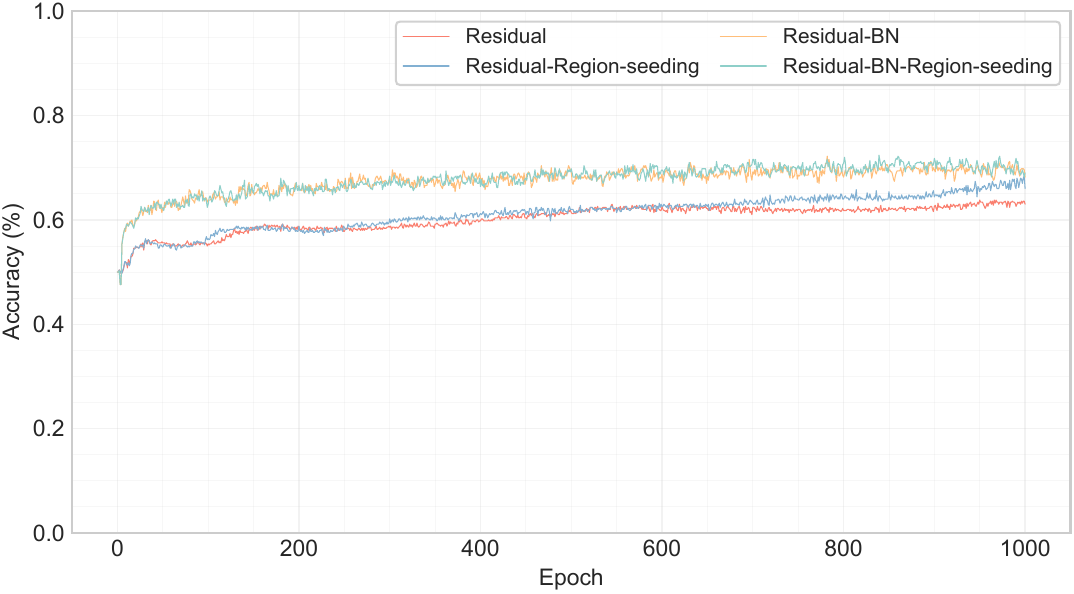}
      \caption{Residual connections}
      \label{fig:toy_acc_res}
    \end{subfigure}
  \end{tabular}
  \caption{Toy dataset: test accuracy over training epochs. Comparison of baseline training and training with the proposed region-seeding regularizer, reported for both architectures under BN and non-BN settings.}
  \label{fig:toy_acc}
\end{figure*}

\begin{figure*}[ht]
  \centering
  \setlength{\tabcolsep}{6pt}
  \renewcommand{\arraystretch}{0}
  \begin{tabular}{@{}cc@{}}
    \begin{subfigure}[t]{0.48\textwidth}
      \centering
      \includegraphics[width=\textwidth]{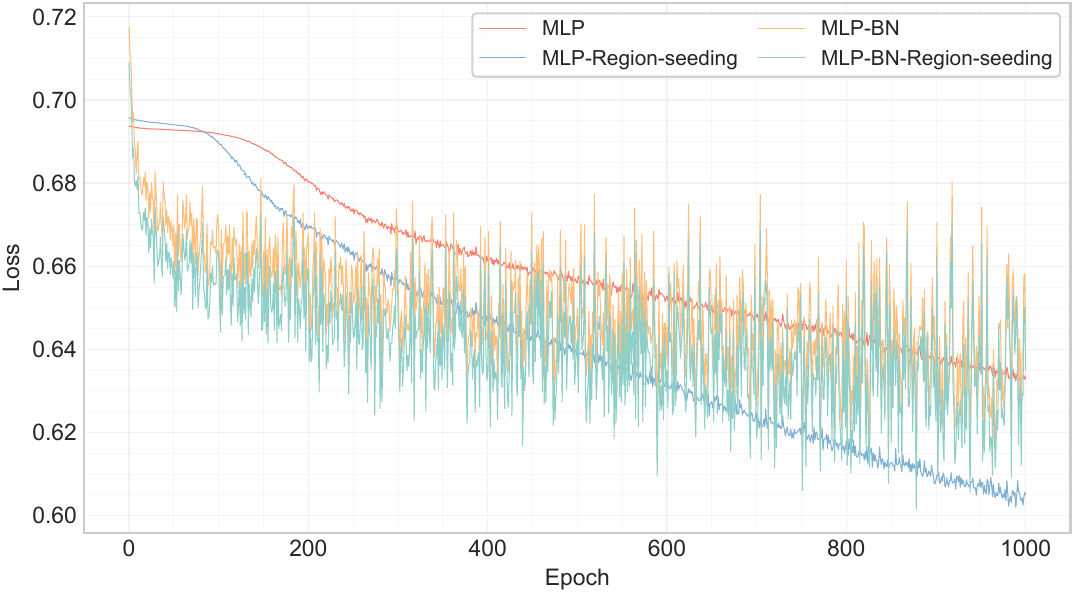}
      \caption{Feedforward MLP}
      \label{fig:toy_loss_mlp}
    \end{subfigure} &
    \begin{subfigure}[t]{0.48\textwidth}
      \centering
      \includegraphics[width=\textwidth]{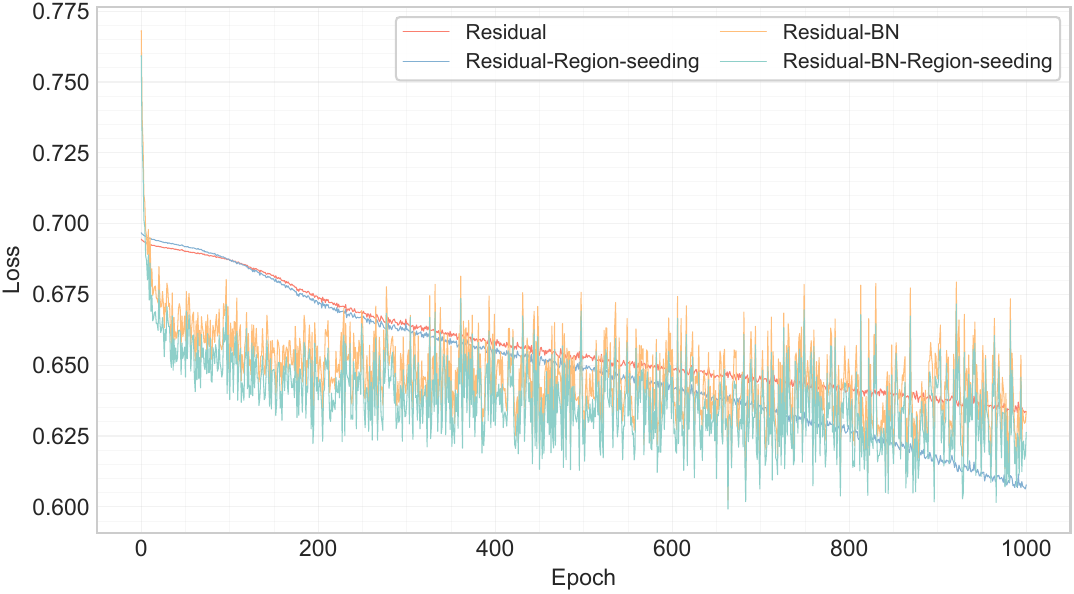}
      \caption{Residual connections}
      \label{fig:toy_loss_res}
    \end{subfigure}
  \end{tabular}
  \caption{Toy dataset: optimization dynamics over training epochs. We plot the task loss excluding the regularization term to assess progress on the classification objective, comparing baseline and region-seeded training for both architectures under BN and non-BN settings.}
  \label{fig:toy_loss}
\end{figure*}

\begin{figure*}[ht]
  \centering
  \setlength{\tabcolsep}{6pt}
  \renewcommand{\arraystretch}{0}
  \begin{tabular}{@{}cc@{}}
    \begin{subfigure}[t]{0.48\textwidth}
      \centering
      \includegraphics[width=\textwidth]{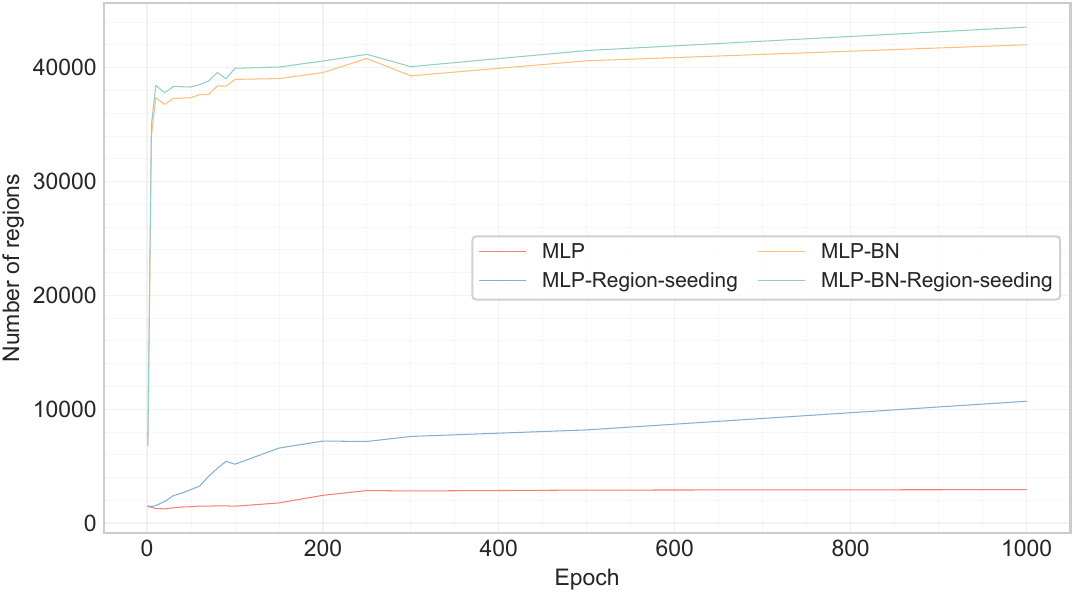}
      \caption{Feedforward MLP}
      \label{fig:toy_regions_mlp}
    \end{subfigure} &
    \begin{subfigure}[t]{0.48\textwidth}
      \centering
      \includegraphics[width=\textwidth]{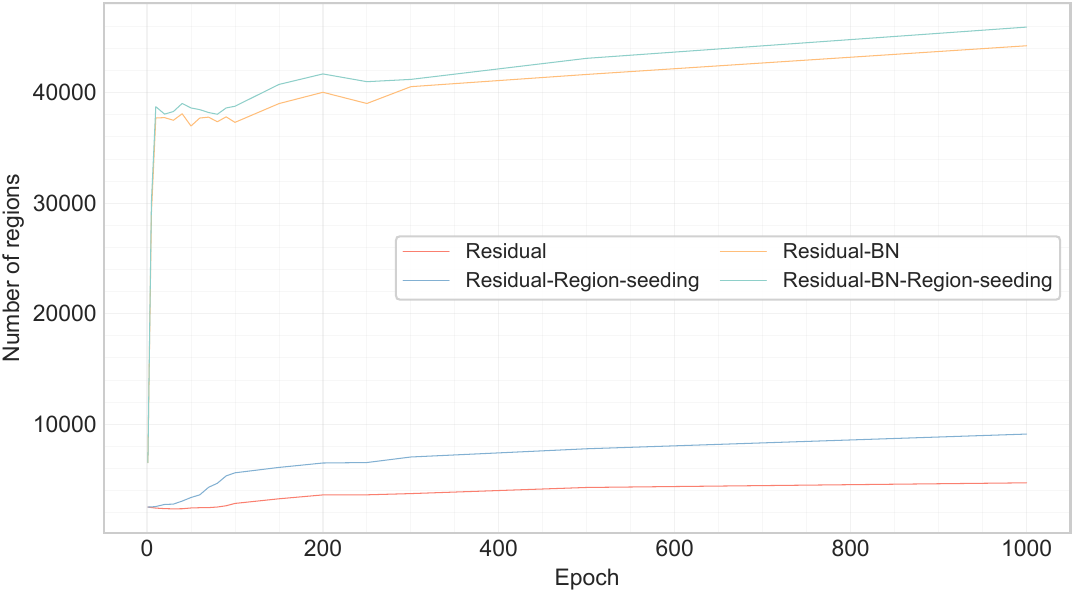}
      \caption{Residual connections}
      \label{fig:toy_regions_res}
    \end{subfigure}
  \end{tabular}
  \caption{Toy dataset: exact realized affine-region counts. Exact enumeration of the number of full-dimensional affine regions induced on $[-1,1]^2$, comparing baseline and region-seeded training for both architectures under BN and non-BN settings.}
  \label{fig:toy_regions}
\end{figure*}

We evaluate region seeding on a controlled synthetic (random~\cite{book38}) binary classification
benchmark intended to isolate how training-time partition seeding affects realized local
expressivity. We use random data to minimize dataset-induced bias toward any particular
geometric shape, so that observed partition refinements more directly reflect the
training-time seeding effect rather than alignment to a specific target structure. The
dataset consists of $500$ samples in $\mathbb{R}^2$, with inputs generated uniformly
over the square domain $[-1,1]^2$ and assigned to two classes (Figure~\ref{fig:toy_regions_vis}).
Unless stated otherwise, we report averages over 3 random seeds. All implementation details and hyperparameters (optimizer, schedule, batch size, annealing design, and any auxiliary settings) are provided in the Appendix~\ref{C}, which also reports full experimental results on two additional toy datasets.

\paragraph{Architectures and training objective.}
We study two continuous piecewise-affine architectures on $\mathbb{R}^2$: (i) a feedforward MLP and (ii) a network with two residual~\cite{book31} connections. Both produce two-class logits via a final linear readout. For each architecture, we evaluate two normalization settings---\emph{without batch normalization (BN)~\cite{book32}} and \emph{with BN}---and, within each (architecture, BN setting) pair, we compare standard training (classification loss only) against training augmented with the proposed region-seeding regularizer (Eq.~\eqref{eq:loss-total}). Concretely, the feedforward MLP has five hidden layers of width $32$, with a linear map, normalization, and ReLU at each hidden layer. The residual-connection architecture uses three width-$32$ layers arranged into two residual stages (two residual connections) before the final linear classifier; in both cases, the only architectural difference between the BN and non-BN variants is the presence or absence of BN after the linear layers. The region-seeding term penalizes the normalized $\ell_2$ magnitude of pre-activation tensors entering each nonlinearity module on the mini-batch (Eq.~\eqref{eq:rtotal}).

\paragraph{Exact affine-region enumeration.}
A key advantage of this low-dimensional setting is that the realized affine partition can
be computed exactly.
For each trained model, we enumerate the full-dimensional cells of the induced CPA
partition restricted to the input domain and report the resulting number of realized
affine regions.
In addition to reporting scalar counts, we visualize the induced input-space partitions
via exact affine-region plots (Figure~\ref{fig:toy_regions_vis}). Specifically,
Figure~\ref{fig:toy_regions_vis} (a) shows the feedforward MLP and
Figure~\ref{fig:toy_regions_vis} (b) shows the residual-connection architecture; within each
row, the composite figure contains five panels (left to right): the toy dataset, followed
by four exact affine-partition visualizations corresponding to the four training settings
given by the presence/absence of BN~\cite{book32} and the presence/absence of the
region-seeding regularizer. For each panel, the number reported in parentheses is the
realized affine-region count at epoch 1000, enabling a direct, architecture-matched
comparison between qualitative partition structure and the exact enumerated region count.

\paragraph{Metrics and results.}
We report (i) test accuracy, (ii) optimization dynamics via the task loss excluding the regularization term (to isolate progress on the classification objective), and (iii) the exact number of realized affine regions on $[-1,1]^2$ obtained by enumeration. Figures~\ref{fig:toy_acc}--\ref{fig:toy_regions} summarize these signals over training epochs, pairing the feedforward MLP and the residual-connection architecture within each metric for direct comparison. Across both architectures and both BN settings, adding region seeding increases the number of realized affine regions (Figure~\ref{fig:toy_regions}) and is accompanied by improved early-stage optimization behavior, reflected in a faster decrease of the task loss (Figure~\ref{fig:toy_loss}) and higher test accuracy across most epochs (Figure~\ref{fig:toy_acc}). These observations are consistent with the proposed mechanism: encouraging small pre-activation magnitudes on data increases the likelihood that neuron switching surfaces intersect data neighborhoods (Theorem~\ref{thm:distance-intersection}), which in turn yields increased local region counts under mild regularity (Theorem~\ref{thm:monotone-cells}). Notably, the qualitative trends persist in the presence of BN, suggesting that the seeding effect is robust to batch-dependent normalization.

\subsection{Real-data experiments}
\label{sec:imagenet}

\begin{figure*}[ht]
\centering
\setlength{\tabcolsep}{2pt} 
\renewcommand{\arraystretch}{0} 
\begin{tabular}{@{}cc@{}}
\includegraphics[width=0.48\textwidth]{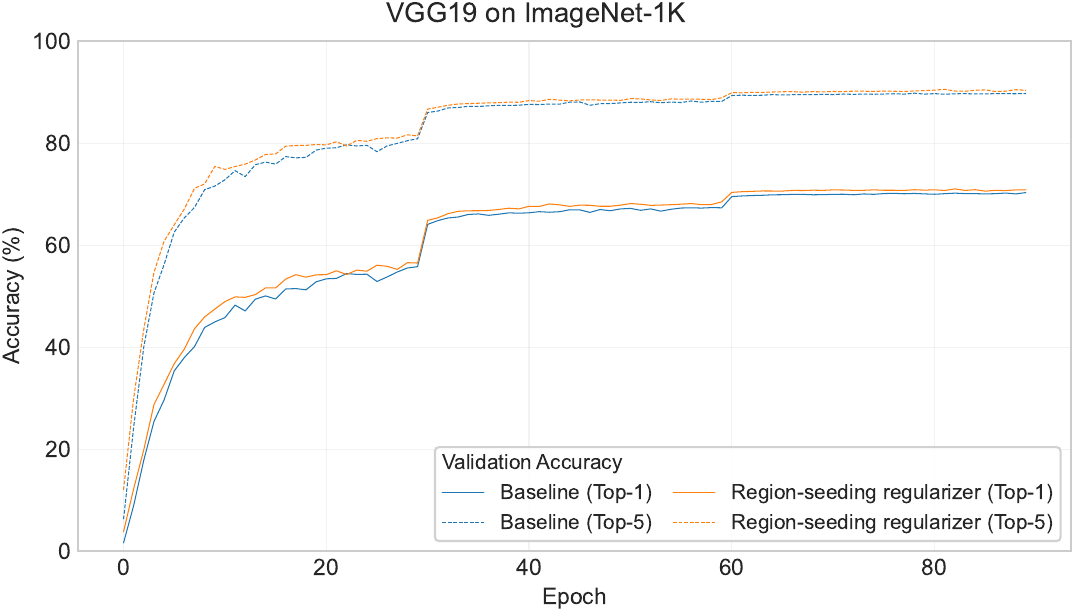} &
\includegraphics[width=0.48\textwidth]{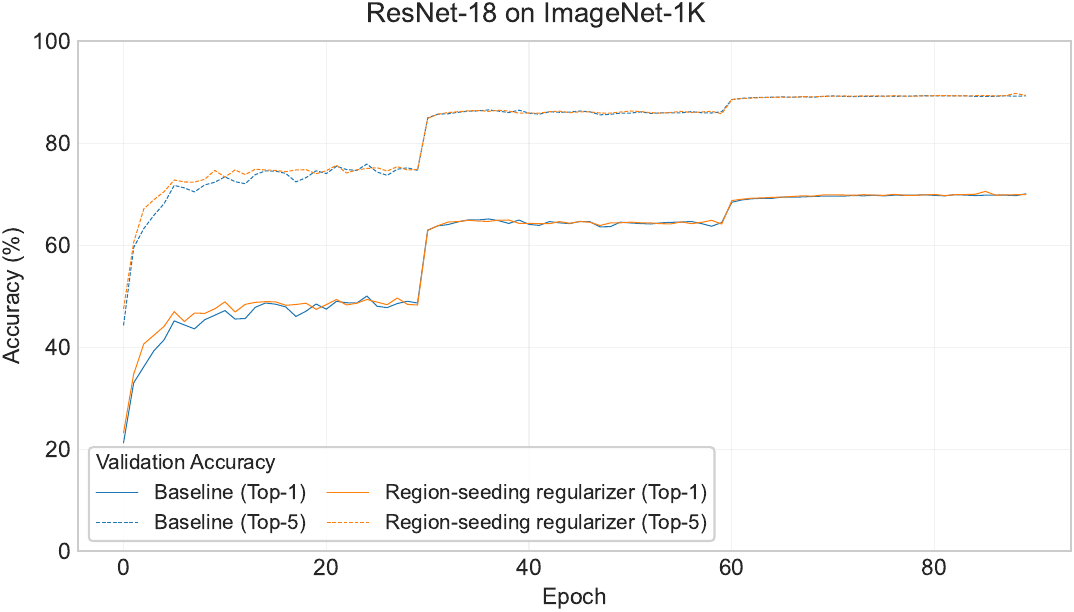} \\
\includegraphics[width=0.48\textwidth]{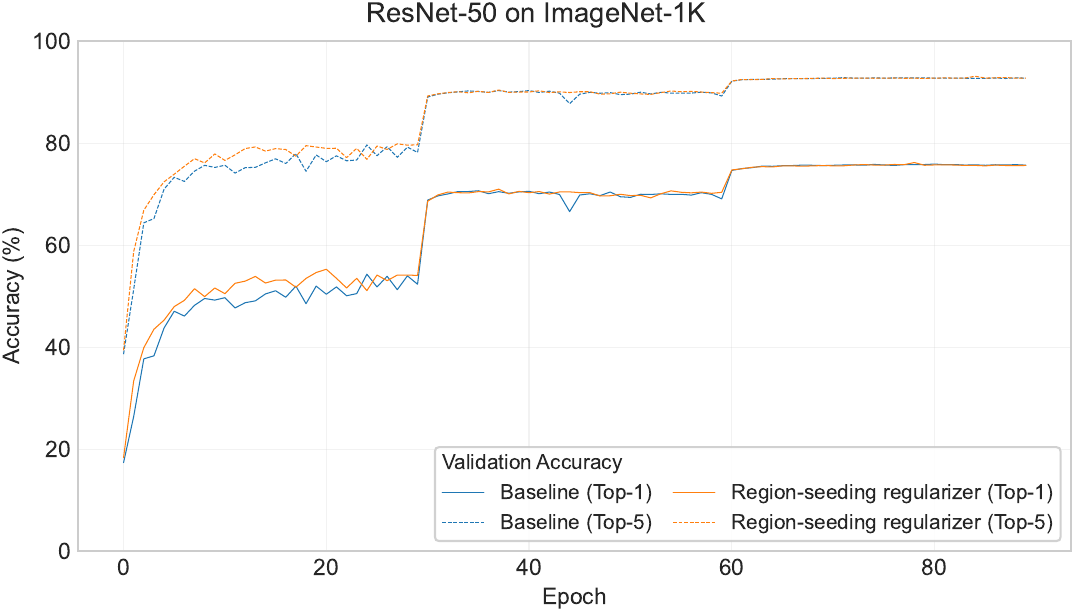} &
\includegraphics[width=0.48\textwidth]{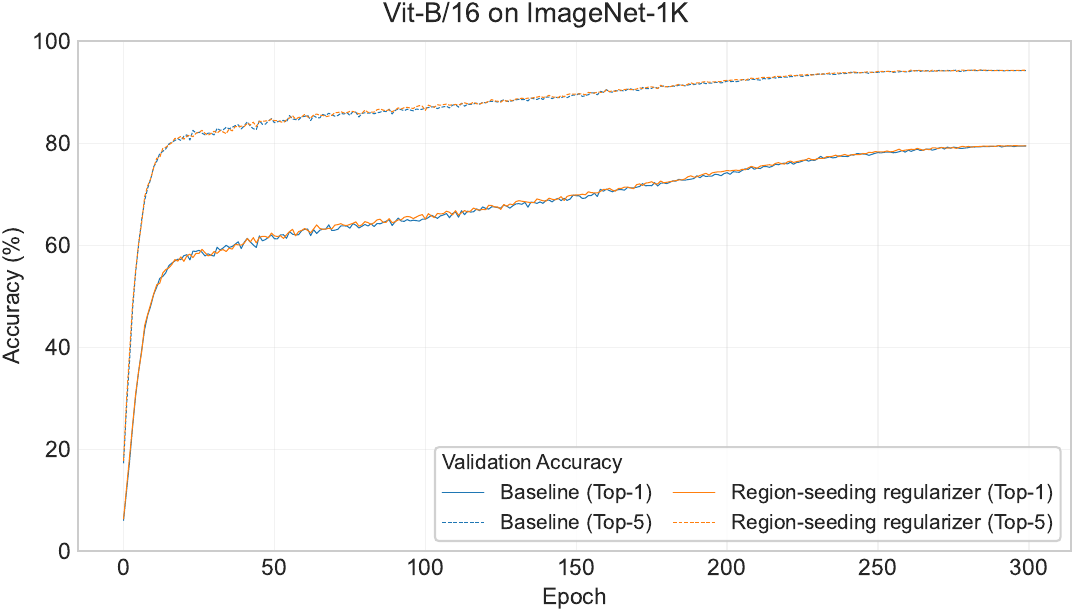} \\
\end{tabular}
\caption{ImageNet-1k validation accuracy trajectories over the full training schedule.
Each panel compares baseline training with training using the region-seeding regularizer,
reporting both Top-1 and Top-5 accuracy.}
\label{fig:imagenet_acc}
\end{figure*}

\begin{figure*}[ht]
\centering
\setlength{\tabcolsep}{2pt} 
\renewcommand{\arraystretch}{0} 
\begin{tabular}{@{}cc@{}}
\includegraphics[width=0.48\textwidth]{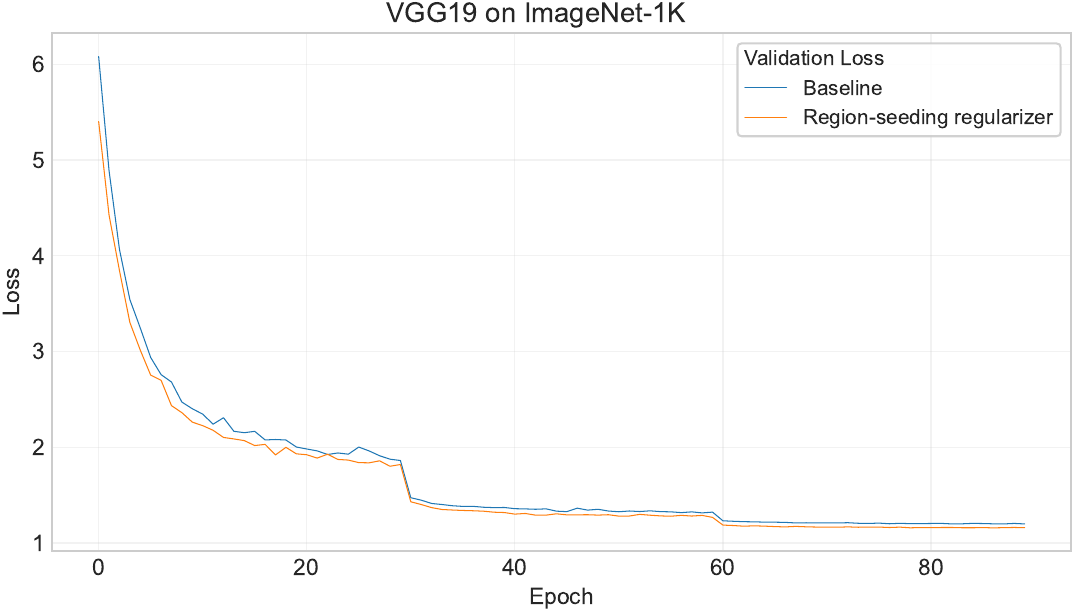} &
\includegraphics[width=0.48\textwidth]{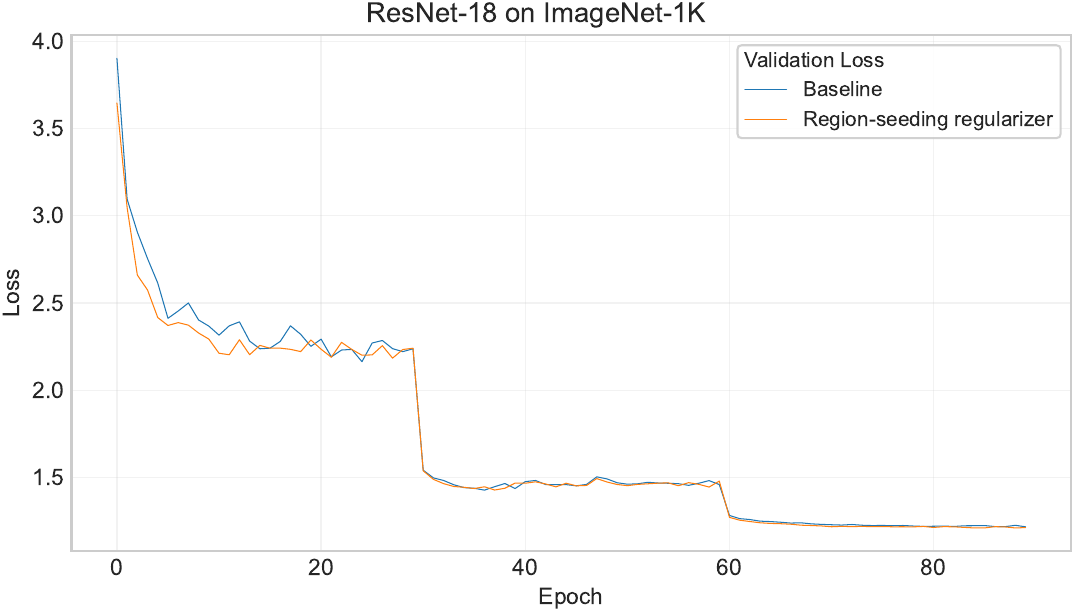} \\
\includegraphics[width=0.48\textwidth]{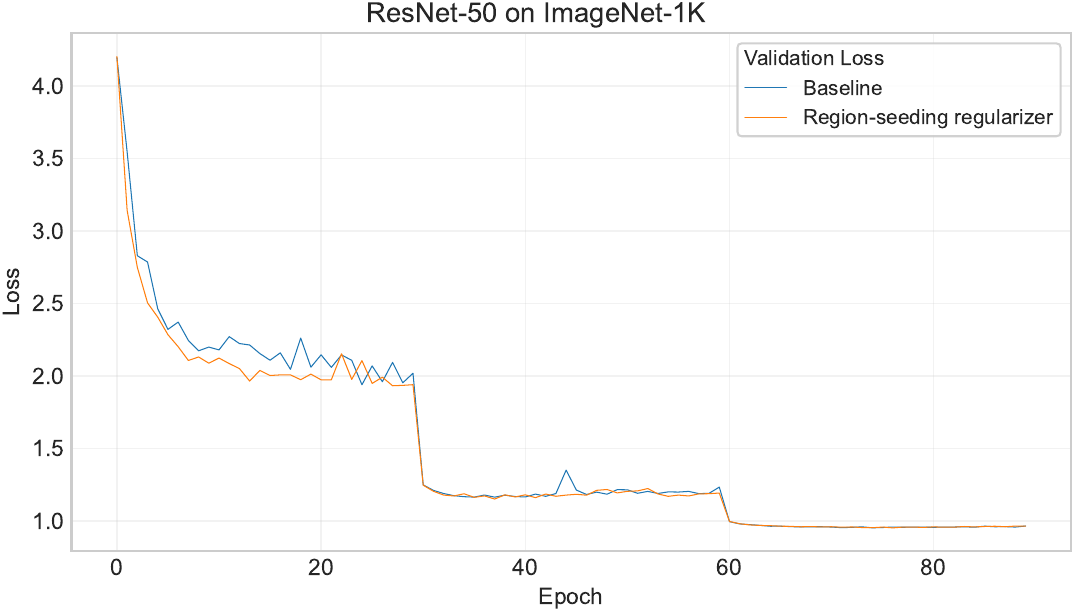} &
\includegraphics[width=0.48\textwidth]{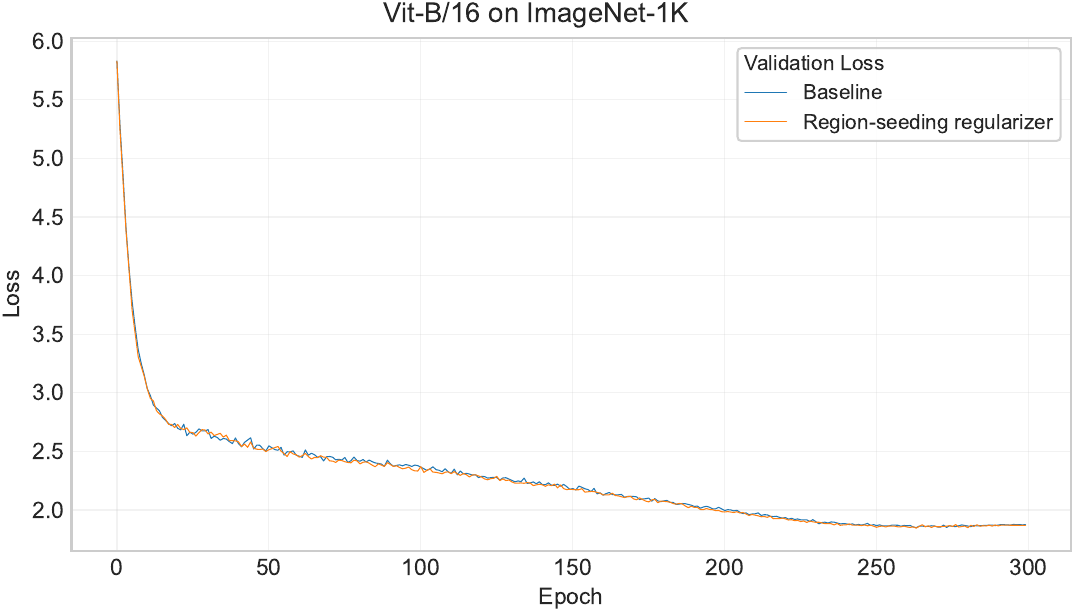} \\
\end{tabular}
\caption{ImageNet-1k optimization dynamics over the full training schedule.
Each panel compares baseline training and region-seeded training.
We plot the task loss (excluding the regularization term) to assess progress on the classification objective.}
\label{fig:imagenet_loss}
\end{figure*}

We evaluate the proposed region-seeding regularizer (Eq.~\eqref{eq:loss-total}) on
ImageNet-1k~\cite{book33} using several widely used architectures. In small-scale toy settings, one
can sometimes exactly enumerate realized affine regions; however, for ImageNet-scale
models such enumeration is typically difficult to carry out in practice and would add
substantial overhead beyond standard training and evaluation. We therefore characterize
the effect of region seeding on ImageNet-1k using standard, directly observed signals:
validation Top-1/Top-5 accuracy trajectories (Figure~\ref{fig:imagenet_acc}) and the
task loss excluding the regularization term (Figure~\ref{fig:imagenet_loss}). Reporting
the task loss without the auxiliary penalty isolates progress on the classification
objective and keeps the loss curves directly comparable between baseline and
regularized runs.
\paragraph{Setup.}
We train on ImageNet-1k using the standard train/validation split and report single-crop
validation Top-1 and Top-5 accuracy. We follow the official PyTorch ImageNet training
recipes for ResNet-18~\cite{book31}, VGG-19~\cite{book34} without BN (VGG-19 noBN), and ResNet-50~\cite{book31}, and additionally
include ViT-B/16~\cite{book35}. For each architecture, baseline and region-seeded runs use identical
data augmentation, optimizer, learning-rate schedule, batch size, weight decay, and
training length; the only change is adding the region-seeding term in
Eq.~\eqref{eq:loss-total}. Unless stated otherwise, we run each configuration with three
random seeds and report mean curves. Hyperparameters, annealing details, and depth-weight
ablations are provided in the Appendix~\ref{D}.

\paragraph{Early-epoch behavior.}
Across VGG-19 noBN, ResNet-18, ResNet-50, and ViT-B/16, the region-seeded runs show higher
validation accuracy than their matched baselines during the early portion of training
(Figure~\ref{fig:imagenet_acc}). This pattern holds for both Top-1 and Top-5 in each
panel, with a visibly consistent separation in the initial epochs. Over the same early
training window, the task-loss curves decrease more rapidly under region-seeded training
than under baseline training (Figure~\ref{fig:imagenet_loss}). Because the plotted loss
excludes the regularization term, this faster decrease reflects faster improvement on
the classification objective as measured by the task loss. The Top-5 trajectories provide
a complementary, ranking-based view of early progress: Top-5 accuracy measures whether
the correct class appears among the five highest-scoring predictions, and the early Top-5
gaps in Figure~\ref{fig:imagenet_acc} therefore correspond directly to the correct label
being placed within the top five more often early in training.

\paragraph{Late-stage behavior.}
As training proceeds, the differences between baseline and region-seeded curves narrow in
both accuracy and loss (Figures~\ref{fig:imagenet_acc} and \ref{fig:imagenet_loss}). By
the end of the training schedule, the region-seeded runs achieve validation performance
that is comparable to the baselines, and in several panels the final Top-1/Top-5 accuracy
is slightly higher for the region-seeded run (Figure~\ref{fig:imagenet_acc}). The task
loss curves likewise approach similar terminal values (Figure~\ref{fig:imagenet_loss}).
This late-stage convergence behavior is consistent with the design of our regularizer:
the region-seeding term is multiplied by an annealing factor $\eta(t)$ that decays over
training (Eq.~\eqref{eq:rtotal}), so its influence is intentionally strongest early and
is reduced later, allowing the task loss to dominate optimization near convergence.

\paragraph{ViT-B/16 as a heuristic extension.}
Our formal region-growth analysis targets networks with continuous piecewise-affine
activations, whereas ViT-B/16 uses smooth nonlinearities. We nevertheless include ViT-B/16
to assess whether the same training-time modification yields similar \emph{observable}
effects on standard ImageNet metrics outside the CPA setting. In the ViT panels of
Figures~\ref{fig:imagenet_acc} and \ref{fig:imagenet_loss}, region-seeded training shows
the same qualitative curve behavior as the CNN models: higher early Top-1/Top-5 accuracy
and a faster early decrease in task loss, followed by comparable (and in the plotted
trajectory, slightly improved) end-of-schedule accuracy. We treat this as an empirical
observation about optimization and validation trajectories and do not attribute it to an
affine-region mechanism.

On ImageNet-1k under matched training recipes for VGG-19 noBN, ResNet-18, ResNet-50, and
ViT-B/16, adding the region-seeding regularizer is associated with (i) higher early-epoch
validation Top-1 and Top-5 accuracy and (ii) faster early reduction of the task loss
(Figures~\ref{fig:imagenet_acc} and Figures~\ref{fig:imagenet_loss}). By the end of training, the
resulting validation accuracy is comparable to the baseline and in several cases slightly
higher.

\section{Conclusion}
We presented a local geometric theory showing that bringing neuron switching surfaces close to data points yields provable growth of local affine-region complexity in CPA networks. Based on this link, we proposed a simple region-seeding regularizer that encourages small pre-activations early in training and is annealed over time. Experiments show increased realized region counts when exact enumeration is possible, improved early-stage optimization and ImageNet-1k accuracy, and comparable final performance. More broadly, the theory suggests alternative training-time objectives that explicitly promote finer local affine partitions around data, beyond the particular regularizer studied here.
\newpage

\vskip 0.2in
\bibliography{sample}

@inproceedings{book1,
  author    = {Guido Mont{\'u}far and Razvan Pascanu and Kyunghyun Cho and Yoshua Bengio},
  title     = {On the Number of Linear Regions of Deep Neural Networks},
  booktitle = {Advances in Neural Information Processing Systems},
  volume    = {27},
  pages     = {2924--2932},
  year      = {2014}
}

@inproceedings{book2,
  author    = {Boris Hanin and David Rolnick},
  title     = {Deep {ReLU} Networks Have Surprisingly Few Activation Patterns},
  booktitle = {Advances in Neural Information Processing Systems},
  volume    = {32},
  pages     = {359--368},
  year      = {2019}
}

@inproceedings{book3,
  author    = {Randall Balestriero and Richard G. Baraniuk},
  title     = {A Spline Theory of Deep Learning},
  booktitle = {Proc. Int. Conf. Mach. Learn. ({ICML})},
  volume    = {80},
  pages     = {374--383},
  year      = {2018},
  publisher = {PMLR}
}

@article{book4,
  title={The Computational Complexity of Counting Linear Regions in ReLU Neural Networks},
  author={Stargalla, Moritz and Hertrich, Christoph and Reichman, Daniel},
  journal = {Advances in Neural Information Processing Systems},
  year={2025}
}

@inproceedings{book5,
  title={Rectified linear units improve restricted boltzmann machines},
  author={Nair, Vinod and Hinton, Geoffrey E},
  booktitle={Proceedings of the 27th international conference on machine learning (ICML-10)},
  pages={807--814},
  year={2010}
}

@inproceedings{book6,
  author    = {Andrew L. Maas and Awni Y. Hannun and Andrew Y. Ng},
  title     = {Rectifier Nonlinearities Improve Neural Network Acoustic Models},
  booktitle = {Proc. {ICML} Workshop on Deep Learning for Audio, Speech and Language Processing},
  year      = {2013},
  address   = {Atlanta, GA, USA}
}

@article{book7,
  title={Gaussian Error Linear Units (Gelus)},
  author={Hendrycks, D},
  journal={arXiv preprint arXiv:1606.08415},
  year={2016}
}

@inproceedings{book8,
  author    = {Hanna Tseran and Guido Mont{\'u}far},
  title     = {On the Expected Complexity of Maxout Networks},
  booktitle = {Advances in Neural Information Processing Systems},
  volume    = {34},
  pages     = {28995--29008},
  year      = {2021}
}

@inproceedings{book9,
  author    = {Arturs Berzins},
  title     = {Polyhedral Complex Extraction from {ReLU} Networks Using Edge Subdivision},
  booktitle = {Proc. Int. Conf. Mach. Learn. ({ICML})},
  volume    = {202},
  pages     = {2234--2244},
  year      = {2023},
  publisher = {PMLR}
}

@inproceedings{book10,
  author    = {Ahmed Imtiaz Humayun and Randall Balestriero and Guha Balakrishnan and Richard G. Baraniuk},
  title     = {SplineCam: Exact Visualization and Characterization of Deep Network Geometry and Decision Boundaries},
  booktitle = {Proc. {IEEE/CVF} Conf. Comput. Vis. Pattern Recognit. ({CVPR})},
  pages     = {3789--3798},
  year      = {2023},
  publisher = {{IEEE}},
  doi       = {10.1109/CVPR52729.2023.00369}
}

@article{book11,
  author  = {Hao Chen and Yu Guang Wang and Huan Xiong},
  title   = {Lower and Upper Bounds for Numbers of Linear Regions of Graph Convolutional Networks},
  journal = {Neural Networks},
  volume  = {168},
  pages   = {394--404},
  year    = {2023},
  doi     = {10.1016/j.neunet.2023.09.025}
}

@article{book12,
  author  = {Huan Xiong and Lei Huang and Wenston J. T. Zang and Xiantong Zhen and Guo-Sen Xie and Bin Gu and Le Song},
  title   = {On the Number of Linear Regions of Convolutional Neural Networks With Piecewise Linear Activations},
  journal = {{IEEE} Trans. Pattern Anal. Mach. Intell.},
  volume  = {46},
  number  = {7},
  pages   = {5131--5148},
  year    = {2024},
  doi     = {10.1109/TPAMI.2024.3361155}
}

@inproceedings{book13,
  author    = {Pawel Piwek and Adam Klukowski and Tianyang Hu},
  title     = {Exact Count of Boundary Pieces of {ReLU} Classifiers: Towards the Proper Complexity Measure for Classification},
  booktitle = {Proc. Conf. Uncertainty in Artificial Intelligence ({UAI})},
  volume    = {216},
  pages     = {1673--1683},
  year      = {2023},
  publisher = {PMLR}
}

@inproceedings{book14,
  author    = {Niket Patel and Guido Mont{\'u}far},
  title     = {On the Local Complexity of Linear Regions in Deep {ReLU} Networks},
  booktitle = {Proc. Int. Conf. Mach. Learn. ({ICML})},
  year      = {2025},
  publisher = {OpenReview.net}
}

@inproceedings{book15,
  author    = {Setareh Cohan and Nam Hee Kim and David Rolnick and Michiel van de Panne},
  title     = {Understanding the Evolution of Linear Regions in Deep Reinforcement Learning},
  booktitle = {Advances in Neural Information Processing Systems},
  volume    = {35},
  year      = {2022}
}

@article{book16,
  author  = {J. Elisenda Grigsby and Kathryn Lindsey},
  title   = {On Transversality of Bent Hyperplane Arrangements and the Topological Expressiveness of {ReLU} Neural Networks},
  journal = {{SIAM} J. Appl. Algebra Geom.},
  volume  = {6},
  number  = {2},
  pages   = {216--242},
  year    = {2022},
  doi     = {10.1137/20M1368902}
}

@article{book17,
  author  = {Alexis Goujon and Arian Etemadi and Michael Unser},
  title   = {On the Number of Regions of Piecewise Linear Neural Networks},
  journal = {J. Comput. Appl. Math.},
  volume  = {441},
  pages   = {115667},
  year    = {2024},
  doi     = {10.1016/j.cam.2023.115667}
}

@inproceedings{book18,
  author    = {Boris Hanin and Ryan S. Jeong and David Rolnick},
  title     = {Deep {ReLU} Networks Preserve Expected Length},
  booktitle = {Proc. Int. Conf. Learn. Represent. ({ICLR})},
  year      = {2022},
  publisher = {OpenReview.net}
}

@inproceedings{book19,
  author    = {Yuan Wang},
  title     = {Estimation and Comparison of Linear Regions for {ReLU} Networks},
  booktitle = {Proc. Int. Joint Conf. Artif. Intell. ({IJCAI})},
  pages     = {3544--3550},
  year      = {2022},
  publisher = {ijcai.org},
  doi       = {10.24963/ijcai.2022/492}
}

@inproceedings{book20,
  author    = {Saket Tiwari and George Konidaris},
  title     = {Effects of Data Geometry in Early Deep Learning},
  booktitle = {Advances in Neural Information Processing Systems},
  volume    = {35},
  year      = {2022}
}

@inproceedings{book21,
  author    = {David Rolnick and Konrad P. Kording},
  title     = {Reverse-Engineering Deep {ReLU} Networks},
  booktitle = {Proc. Int. Conf. Mach. Learn. ({ICML})},
  volume    = {119},
  pages     = {8178--8187},
  year      = {2020},
  publisher = {PMLR}
}

@inproceedings{book22,
  author    = {Xiao Zhang and Dongrui Wu},
  title     = {Empirical Studies on the Properties of Linear Regions in Deep Neural Networks},
  booktitle = {Proc. Int. Conf. Learn. Represent. ({ICLR})},
  year      = {2020},
  publisher = {OpenReview.net}
}

@inproceedings{book23,
  author    = {Martin Trimmel and Henning Petzka and Cristian Sminchisescu},
  title     = {TropEx: An Algorithm for Extracting Linear Terms in Deep Neural Networks},
  booktitle = {Proc. Int. Conf. Learn. Represent. ({ICLR})},
  year      = {2021},
  publisher = {OpenReview.net}
}

@article{book24,
  author  = {Juncai He and Lin Li and Jinchao Xu and Chunyue Zheng},
  title   = {{ReLU} Deep Neural Networks and Linear Finite Elements},
  journal = {Journal of Computational Mathematics},
  volume  = {38},
  number  = {3},
  pages   = {502--527},
  year    = {2020},
  doi     = {10.4208/jcm.1901-m2018-0160}
}

@article{book25,
  author  = {Christoph Hertrich and Amitabh Basu and Marco Di Summa and Martin Skutella},
  title   = {Towards Lower Bounds on the Depth of {ReLU} Neural Networks},
  journal = {{SIAM} J. Discret. Math.},
  volume  = {37},
  number  = {2},
  pages   = {997--1029},
  year    = {2023},
  doi     = {10.1137/22M1489332}
}

@inproceedings{book26,
  author    = {Kuan-Lin Chen and Harinath Garudadri and Bhaskar D. Rao},
  title     = {Improved Bounds on Neural Complexity for Representing Piecewise Linear Functions},
  booktitle = {Advances in Neural Information Processing Systems},
  volume    = {35},
  year      = {2022}
}

@inproceedings{book27,
  author    = {Christian Haase and Christoph Hertrich and Georg Loho},
  title     = {Lower Bounds on the Depth of Integral {ReLU} Neural Networks via Lattice Polytopes},
  booktitle = {Proc. Int. Conf. Learn. Represent. ({ICLR})},
  year      = {2023},
  publisher = {OpenReview.net}
}

@inproceedings{book28,
  author    = {Gennadiy Averkov and Christopher Hojny and Maximilian Merkert},
  title     = {On the Expressiveness of Rational {ReLU} Neural Networks With Bounded Depth},
  booktitle = {Proc. Int. Conf. Learn. Represent. ({ICLR})},
  year      = {2025},
  publisher = {OpenReview.net}
}

@article{book29,
  author  = {Qiang Hu and Hao Zhang and Feifei Gao and Chengwen Xing and Jianping An},
  title   = {Analysis on the Number of Linear Regions of Piecewise Linear Neural Networks},
  journal = {{IEEE} Trans. Neural Netw. Learn. Syst.},
  volume  = {33},
  number  = {2},
  pages   = {644--653},
  year    = {2022},
  doi     = {10.1109/TNNLS.2020.3028431}
}

@inproceedings{book30,
  author    = {Max Milkert and David Hyde and Forrest J. Laine},
  title     = {Compelling {ReLU} Networks to Exhibit Exponentially Many Linear Regions at Initialization and During Training},
  booktitle = {Proc. Int. Conf. Mach. Learn. ({ICML})},
  year      = {2025},
  publisher = {OpenReview.net}
}

@inproceedings{book31,
  title={Deep residual learning for image recognition},
  author={He, Kaiming and Zhang, Xiangyu and Ren, Shaoqing and Sun, Jian},
  booktitle={Proceedings of the IEEE conference on computer vision and pattern recognition},
  pages={770--778},
  year={2016}
}

@inproceedings{book32,
  title={Batch normalization: Accelerating deep network training by reducing internal covariate shift},
  author={Ioffe, Sergey and Szegedy, Christian},
  booktitle={International conference on machine learning},
  pages={448--456},
  year={2015},
  organization={pmlr}
}

@inproceedings{book33,
  title={Imagenet: A large-scale hierarchical image database},
  author={Deng, Jia and Dong, Wei and Socher, Richard and Li, Li-Jia and Li, Kai and Fei-Fei, Li},
  booktitle={2009 IEEE conference on computer vision and pattern recognition},
  pages={248--255},
  year={2009},
  organization={Ieee}
}

@article{book34,
  title={Very deep convolutional networks for large-scale image recognition},
  author={Simonyan, Karen and Zisserman, Andrew},
  journal={arXiv preprint arXiv:1409.1556},
  year={2014}
}

@article{book35,
  title={An image is worth 16x16 words: Transformers for image recognition at scale},
  author={Dosovitskiy, Alexey},
  journal={arXiv preprint arXiv:2010.11929},
  year={2020}
}

@article{book36,
  title={Deep learning},
  author={LeCun, Yann and Bengio, Yoshua and Hinton, Geoffrey},
  journal={nature},
  volume={521},
  number={7553},
  pages={436--444},
  year={2015},
  publisher={Nature Publishing Group UK London}
}

@book{book37,
  title={Arrangements of hyperplanes},
  author={Orlik, Peter and Terao, Hiroaki},
  volume={300},
  year={2013},
  publisher={Springer Science \& Business Media}
}

@article{book38,
  title={Scikit-learn: Machine learning in Python},
  author={Pedregosa, Fabian and Varoquaux, Ga{\"e}l and Gramfort, Alexandre and Michel, Vincent and Thirion, Bertrand and Grisel, Olivier and Blondel, Mathieu and Prettenhofer, Peter and Weiss, Ron and Dubourg, Vincent and others},
  journal={the Journal of machine Learning research},
  volume={12},
  pages={2825--2830},
  year={2011},
  publisher={JMLR. org}
}

@inproceedings{book39,
  title={Comparative analysis of the linear regions in ReLU and LeakyReLU networks},
  author={Qi, Xuan and Wei, Yi and Mei, Xue and Chellali, Ryad and Yang, Shipin},
  booktitle={International Conference on Neural Information Processing},
  pages={528--539},
  year={2023},
  organization={Springer}
}

@inproceedings{book40,
  title={Empirical Study on the Effect of Residual Networks on the Expressiveness of Linear Regions},
  author={Qi, Xuan and Wei, Yi and Mei, Xue and Chellali, Ryad and Yang, Shipin},
  booktitle={International Conference on Artificial Neural Networks},
  pages={174--185},
  year={2023},
  organization={Springer}
}

@article{book41,
  title={The Evolution of the Interplay Between Input Distributions and Linear Regions in Networks},
  author={Qi, Xuan and Wei, Yi},
  journal={arXiv preprint arXiv:2310.18725},
  year={2023}
}
\clearpage
\appendix

\section{Proof of Theorem~\ref{thm:monotone-cells}}
\label{Proof1}
\begin{assumption}[Hereditary local non-degeneracy inside $P_\varepsilon(x)$]
\label{assump:local-nondegeneracy}
Let $P:=P_\varepsilon(x)\subset\mathbb{R}^d$ be a convex closed polytope with
$\mathrm{int}(P)\neq\varnothing$.
Let $\mathcal{J}$ be a finite index set and let $\{z_j\}_{j\in\mathcal{J}}$ be
continuous piecewise-affine (CPA) functions. For each $j$, write
$Z_j:=\{y\in\mathbb{R}^d:\ z_j(y)=0\}$.

We say $\mathcal{J}$ satisfies \emph{hereditary local non-degeneracy} inside $P$
if the following holds for every subfamily $\mathcal{K}\subseteq\mathcal{J}$.

\begin{enumerate}
\item[\textbf{(A1)}] \textbf{Finite CPA complexity on $P$.}
Each $z_j$ is CPA on an open neighborhood of $P$.
Moreover, there exists a finite polyhedral subdivision of $P$ into
full-dimensional open polytopes such that, on each polytope, $z_j$ is affine.

\item[\textbf{(A2)}] \textbf{Regular (non-fold) traces that cut the interior.}
For every $j\in\mathcal{K}$ and every cell $R\in\mathcal{P}_{\mathcal{K}\setminus\{j\}}$ with
$R\cap \mathrm{int}(P)\neq\varnothing$, the set
$Z_j\cap R\cap \mathrm{int}(P)$ is either empty, or else it is a nonempty
$(d-1)$-dimensional affine hyperplane patch $\Gamma_{j,R}$ satisfying:

\begin{enumerate}
\item[\textbf{(i)}] \textbf{Interior placement.}
\begin{equation}
\mathrm{relint}(\Gamma_{j,R})\subset \mathrm{int}\bigl(R\cap P\bigr).
\end{equation}

\item[\textbf{(ii)}] \textbf{Local single-affine representation with nonzero normal.}
There exist $a_{j,R}\in\mathbb{R}^d\setminus\{0\}$, $c_{j,R}\in\mathbb{R}$, and an
open neighborhood $U_{j,R}\subset R\cap \mathrm{int}(P)$ of
$\mathrm{relint}(\Gamma_{j,R})$ such that
\begin{align}
z_j(y) &= a_{j,R}^\top y + c_{j,R},
&& \forall\,y\in U_{j,R},\\
\Gamma_{j,R}\cap U_{j,R}
&= \bigl\{y\in U_{j,R}:\ a_{j,R}^\top y + c_{j,R}=0\bigr\}.
\end{align}

In particular, $\Gamma_{j,R}$ is not a fold trace (e.g., it excludes
$z(y)=|y_1|$ along $\{y_1=0\}$), and the trace locally separates $U_{j,R}$ into
two nonempty strict-sign sides.
\end{enumerate}

\item[\textbf{(A3)}] \textbf{Non-redundancy of intersecting traces (essentiality).}
For every $j\in\mathcal{K}$ with $Z_j\cap \mathrm{int}(P)\neq\varnothing$,
there exists a point
\begin{equation}
y\in Z_j\cap \mathrm{int}(P)
\quad\text{such that}\quad
z_k(y)\neq 0\ \ \forall\,k\in\mathcal{K}\setminus\{j\}.
\end{equation}

Equivalently,
\begin{equation}
Z_j\cap \mathrm{int}(P)
\not\subseteq
\bigcup_{k\in\mathcal{K}\setminus\{j\}} Z_k.
\end{equation}

\end{enumerate}
\end{assumption}

\begin{lemma}[Counting cells intersecting $P$ vs.\ $\mathrm{int}(P)$]
\label{lem:N-eq-Ncirc}
Let $P\subset\mathbb{R}^d$ be a convex closed polytope with $\mathrm{int}(P)\neq\varnothing$, and let
$\mathcal{P}$ be any collection of full-dimensional open cells. Define
\begin{equation}
N(P;\mathcal{P}) := \#\{R\in\mathcal{P}:\; R\cap P\neq\varnothing\},\qquad
N^\circ(P;\mathcal{P}) := \#\{R\in\mathcal{P}:\; R\cap \mathrm{int}(P)\neq\varnothing\}.
\end{equation}

Then $N(P;\mathcal{P})=N^\circ(P;\mathcal{P})$.
\end{lemma}

\begin{proof}
Trivially $N^\circ(P;\mathcal{P})\le N(P;\mathcal{P})$ since $\mathrm{int}(P)\subseteq P$.
For the reverse direction, fix an open full-dimensional cell $R\in\mathcal{P}$ with $R\cap P\neq\varnothing$,
and choose $y\in R\cap P$. Since $R$ is open, there exists $r>0$ such that $B(y,r)\subset R$.
If $y\in \mathrm{int}(P)$ we are done. If $y\in\partial P$, then every neighborhood of $y$ intersects
$\mathrm{int}(P)$, hence $B(y,r)\cap \mathrm{int}(P)\neq\varnothing$ and thus $R\cap \mathrm{int}(P)\neq\varnothing$.
Therefore $N(P;\mathcal{P})\le N^\circ(P;\mathcal{P})$.
\end{proof}


\begin{proof}[Proof of Theorem~\ref{thm:monotone-cells}]
We first argue for a general convex closed polytope $P\subset\mathbb{R}^d$ with $\mathrm{int}(P)\neq\varnothing$,
and then specialize to $P=P_\varepsilon(x)$.

\paragraph{Step 0 (Notation and basic facts).}
For any index family $\mathcal{J}$, let $\mathcal{P}_{\mathcal{J}}$ denote the sign partition of $\mathbb{R}^d$
induced by $\{z_j\}_{j\in\mathcal{J}}$ (as in the theorem statement). Define
\begin{equation}
\mathcal{S}(P;\mathcal{J}) := \{j\in\mathcal{J}:\; Z_j\cap \mathrm{int}(P)\neq\varnothing\},
\qquad
I(P;\mathcal{J}) := |\mathcal{S}(P;\mathcal{J})|.
\end{equation}

and
\begin{equation}
N(P;\mathcal{J}) := \#\{R\in\mathcal{P}_{\mathcal{J}}:\; R\cap P\neq\varnothing\}.
\end{equation}

By Lemma~\ref{lem:N-eq-Ncirc}, we may equivalently count intersections with $\mathrm{int}(P)$:
\begin{equation}
\label{eq:N-intP}
N(P;\mathcal{J}) = \#\{R\in\mathcal{P}_{\mathcal{J}}:\; R\cap \mathrm{int}(P)\neq\varnothing\}.
\end{equation}

We also record a monotonicity property under hereditary non-degeneracy.
Assume additionally that $\mathcal{J}'$ satisfies Assumption~\ref{assump:local-nondegeneracy}
inside $P$. If $\mathcal{J}\subseteq\mathcal{J}'$, then $\mathcal{P}_{\mathcal{J}'}$ refines
$\mathcal{P}_{\mathcal{J}}$ in the sense that every $R'\in\mathcal{P}_{\mathcal{J}'}$ is contained
in some $R\in\mathcal{P}_{\mathcal{J}}$ (since fixing signs for $\mathcal{J}'$ fixes signs for $\mathcal{J}$).
Consequently,
\begin{equation}
\label{eq:monotone-N}
N(P;\mathcal{J}')\ \ge\ N(P;\mathcal{J}).
\end{equation}
Indeed, fix any $R\in\mathcal{P}_{\mathcal{J}}$ with $R\cap \mathrm{int}(P)\neq\varnothing$.
Because $\mathcal{J}'$ is hereditarily non-degenerate, for each $j\in\mathcal{J}'\setminus\mathcal{J}$,
applying Assumption~\ref{assump:local-nondegeneracy}(A2) to the singleton subfamily $\mathcal{K}=\{j\}$
shows that $Z_j\cap \mathrm{int}(P)$ is either empty or a $(d-1)$-dimensional affine hyperplane patch,
hence has empty interior in $\mathrm{int}(P)$. Therefore the finite union
$\bigcup_{j\in\mathcal{J}'\setminus\mathcal{J}} \bigl(Z_j\cap \mathrm{int}(P)\bigr)$ has empty interior in
$R\cap \mathrm{int}(P)$, so we can choose
\begin{equation}
y\in \bigl(R\cap \mathrm{int}(P)\bigr)\setminus \bigcup_{j\in\mathcal{J}'\setminus\mathcal{J}} Z_j.
\end{equation}

Then $\mathrm{sign}(z_j(y))$ is well-defined for all $j\in\mathcal{J}'$, hence $y$ belongs to a unique
$R'\in\mathcal{P}_{\mathcal{J}'}$, and necessarily $R'\subseteq R$ and $R'\cap \mathrm{int}(P)\neq\varnothing$.
Thus each intersecting $R$ contributes at least one intersecting $R'$, proving \eqref{eq:monotone-N}.

\paragraph{Step 1 (One additional zero-set intersection increases the count by at least $1$).}
Let $j^\star\notin\mathcal{J}$ and set $\mathcal{J}':=\mathcal{J}\cup\{j^\star\}$.
Assume there exists a cell $R\in\mathcal{P}_{\mathcal{J}}$ such that
$Z_{j^\star}\cap R\cap \mathrm{int}(P)\neq\varnothing$ and there exist
$y^\pm\in R\cap \mathrm{int}(P)$ with $\pm z_{j^\star}(y^\pm)>0$.
Assume additionally that $\mathcal{J}'$ satisfies Assumption~\ref{assump:local-nondegeneracy} inside $P$.

Set
\begin{equation}
U := R\cap \mathrm{int}(P),
\qquad
U^+ := \{y\in U:\ z_{j^\star}(y)>0\},
\qquad
U^- := \{y\in U:\ z_{j^\star}(y)<0\}.
\end{equation}

Since $U$ is open and $z_{j^\star}$ is continuous, $U^+$ and $U^-$ are open in $U$ and disjoint.
By assumption, $z_{j^\star}$ takes both positive and negative values on $U$, hence
$U^+$ and $U^-$ are both nonempty.

We now show that $U$ intersects at least two distinct $\mathcal{J}'$-cells.
By definition of $\mathcal{P}_{\mathcal{J}}$, the signs of $\{z_j\}_{j\in\mathcal{J}}$ are constant on $R$,
hence constant on $U\subset R$. Therefore, within $U$ the only additional sign constraint distinguishing
$\mathcal{P}_{\mathcal{J}'}$ from $\mathcal{P}_{\mathcal{J}}$ is the sign of $z_{j^\star}$; in particular,
every $\mathcal{J}'$-cell contained in $R$ lies entirely in either $U^+$ or $U^-$.

Pick $y^+\in U^+$ and $y^-\in U^-$. Each of $y^+,y^-$ belongs to a unique cell in $\mathcal{P}_{\mathcal{J}'}$,
call them $R^+,R^-\in\mathcal{P}_{\mathcal{J}'}$. Necessarily $R^+\subseteq U^+$ and $R^-\subseteq U^-$, so
$R^+\neq R^-$. Moreover, $R^\pm\cap \mathrm{int}(P)\neq\varnothing$ because $y^\pm\in \mathrm{int}(P)$.

Thus, the single $\mathcal{J}$-cell $R$ that intersects $\mathrm{int}(P)$ contributes at least two distinct
$\mathcal{J}'$-cells intersecting $\mathrm{int}(P)$, while every other $\mathcal{J}$-cell intersecting
$\mathrm{int}(P)$ contributes at least one such $\mathcal{J}'$-cell (by \eqref{eq:monotone-N}).
Therefore,
\begin{equation}
N(P;\mathcal{J}')\ \ge\ N(P;\mathcal{J})+1.
\end{equation}

which proves Theorem~\ref{thm:monotone-cells}\textbf{(i)} for a general polytope $P$.

\paragraph{Step 2 (Lower bound $N\ge 1+I$).}
Let $\mathcal{S}(P;\mathcal{J})=\{j_1,\dots,j_I\}$ with $I:=I(P;\mathcal{J})$. Consider the increasing sequence
\begin{equation}
\mathcal{J}^{(0)}:=\varnothing,\qquad
\mathcal{J}^{(t)}:=\{j_1,\dots,j_t\}\quad (t=1,\dots,I).
\end{equation}

Clearly, $\mathcal{P}_{\mathcal{J}^{(0)}}=\{\mathbb{R}^d\}$ and hence $N(P;\mathcal{J}^{(0)})=1$.

We claim that for each $t=1,\dots,I$,
\begin{equation}
\label{eq:induction-step-strong}
N(P;\mathcal{J}^{(t)})\ \ge\ N(P;\mathcal{J}^{(t-1)})+1.
\end{equation}
Fix such $t$ and write $\widetilde{\mathcal{J}}:=\mathcal{J}^{(t-1)}$ and $j^\star:=j_t$, so that
$\mathcal{J}^{(t)}=\widetilde{\mathcal{J}}\cup\{j^\star\}$.

Because $j^\star\in \mathcal{S}(P;\mathcal{J})$, we have $Z_{j^\star}\cap \mathrm{int}(P)\neq\varnothing$.
Apply Assumption~\ref{assump:local-nondegeneracy}(A3) to the subfamily $\mathcal{K}=\mathcal{J}^{(t)}$:
there exists
\begin{equation}
y\in Z_{j^\star}\cap \mathrm{int}(P)
\quad\text{such that}\quad
z_k(y)\neq 0\ \ \forall\,k\in \widetilde{\mathcal{J}}.
\end{equation}

In particular, $y\in \mathrm{int}(P)$ and $y$ is not on any zero set $Z_k$ with $k\in \widetilde{\mathcal{J}}$.
Therefore, the strict sign pattern $\{\mathrm{sign}(z_k(y))\}_{k\in \widetilde{\mathcal{J}}}$ is well-defined.
Let
\begin{equation}
S_y := \Bigl\{u\in\mathbb{R}^d:\ \mathrm{sign}(z_k(u))=\mathrm{sign}(z_k(y))\ \ \forall k\in \widetilde{\mathcal{J}}\Bigr\}.
\end{equation}

By continuity of each $z_k$ ($k\in\widetilde{\mathcal{J}}$), the set $S_y$ is open, and $y\in S_y$.
Let $R$ be the unique connected component of $S_y$ containing $y$; then $R\in\mathcal{P}_{\widetilde{\mathcal{J}}}$ and
$y\in R\cap \mathrm{int}(P)$.

We now verify the hypothesis of Theorem~\ref{thm:monotone-cells}\textbf{(i)} for the pair
$(\widetilde{\mathcal{J}},j^\star)$. We already have $y\in Z_{j^\star}\cap R\cap \mathrm{int}(P)$, hence
$Z_{j^\star}\cap R\cap \mathrm{int}(P)\neq\varnothing$.

By Assumption~\ref{assump:local-nondegeneracy}(A2) applied to the subfamily
$\mathcal{K}=\widetilde{\mathcal{J}}\cup\{j^\star\}$, the set
$Z_{j^\star}\cap R\cap \mathrm{int}(P)$ is a nonempty $(d-1)$-dimensional affine hyperplane patch
$\Gamma_{j^\star,R}$ whose relative interior lies in $\mathrm{int}(R\cap P)$.
Moreover, Assumption~\ref{assump:local-nondegeneracy}(A2)(ii) gives an open neighborhood
$U_{j^\star,R}\subset R\cap \mathrm{int}(P)$ of $\mathrm{relint}(\Gamma_{j^\star,R})$ on which
$z_{j^\star}$ coincides with a nondegenerate affine function. Hence $z_{j^\star}$ takes both positive and
negative values on $U_{j^\star,R}\subset R\cap \mathrm{int}(P)$, so there exist
$y^\pm\in R\cap \mathrm{int}(P)$ such that $\pm z_{j^\star}(y^\pm)>0$.

Thus the hypotheses of Step~1 hold for the pair $(\widetilde{\mathcal{J}},j^\star)$, and we conclude
\begin{equation}
N(P;\mathcal{J}^{(t)})\ \ge\ N(P;\widetilde{\mathcal{J}})+1\ =\ N(P;\mathcal{J}^{(t-1)})+1.
\end{equation}

which establishes \eqref{eq:induction-step-strong}.

Summing \eqref{eq:induction-step-strong} over $t=1,\dots,I$ yields
\begin{equation}
N(P;\mathcal{J}^{(I)})\ \ge\ 1 + I.
\end{equation}

Finally, $\mathcal{J}^{(I)}=\mathcal{S}(P;\mathcal{J})\subseteq \mathcal{J}$, so by monotonicity \eqref{eq:monotone-N},
\begin{equation}
N(P;\mathcal{J})\ \ge\ N(P;\mathcal{J}^{(I)})\ \ge\ 1+I(P;\mathcal{J}).
\end{equation}

which proves Theorem~\ref{thm:monotone-cells}(ii) for general $P$.

\paragraph{Step 3 (Upper bound under local general position).}
Assume the hypotheses of Theorem~\ref{thm:monotone-cells}(iii).
Let $n:=I(P;\mathcal{J})$ and enumerate the intersecting indices as
$\mathcal{S}(P;\mathcal{J})=\{j_1,\dots,j_n\}$.
For each $r\in\{1,\dots,n\}$, by hypothesis the trace
$T_r := Z_{j_r}\cap \mathrm{int}(P)$ is a single $(d-1)$-dimensional affine
hyperplane patch. Let $H_r$ denote the unique affine hyperplane containing $T_r$.

Let $\mathcal{Q}$ be the (open) region decomposition of $\mathbb{R}^d$ induced by
the affine hyperplanes $\{H_1,\dots,H_n\}$, i.e., the full-dimensional open
connected components of
$\mathbb{R}^d\setminus \bigcup_{r=1}^n H_r$.
Define the local arrangement count
\begin{equation}
N_{\mathrm{arr}}(P) := \#\{Q\in\mathcal{Q}:\ Q\cap \mathrm{int}(P)\neq\varnothing\}.
\end{equation}

\emph{Claim:} the arrangement partition $\mathcal{Q}$ restricted to $\mathrm{int}(P)$
refines the sign partition $\mathcal{P}_{\mathcal{J}}$ restricted to $\mathrm{int}(P)$.
Consequently,
\begin{equation}
\label{eq:N-le-arr}
N(P;\mathcal{J})\ \le\ N_{\mathrm{arr}}(P).
\end{equation}

To prove the claim, fix any $Q\in\mathcal{Q}$ with $Q\cap \mathrm{int}(P)\neq\varnothing$.
Then $Q$ is open and connected, and for each $r$ we have $Q\cap H_r=\varnothing$.
Since $T_r\subseteq H_r$, it follows that $Q\cap T_r=\varnothing$, hence
$z_{j_r}$ is continuous and nonzero on $Q\cap \mathrm{int}(P)$, so its sign is
constant there. For any index $j\in\mathcal{J}\setminus \mathcal{S}(P;\mathcal{J})$,
we have $Z_j\cap \mathrm{int}(P)=\varnothing$, so by continuity $z_j$ has constant
strict sign on the connected set $\mathrm{int}(P)$ (and thus on $Q\cap \mathrm{int}(P)$).
Therefore the full sign pattern $\{\mathrm{sign}(z_j)\}_{j\in\mathcal{J}}$ is constant on
$Q\cap \mathrm{int}(P)$, which implies that $Q\cap \mathrm{int}(P)$ is contained
in a single sign cell $R\in\mathcal{P}_{\mathcal{J}}$.
Hence each such $Q$ lies inside some $R$, i.e., the arrangement partition refines
the sign partition on $\mathrm{int}(P)$, proving \eqref{eq:N-le-arr}.

It remains to bound $N_{\mathrm{arr}}(P)$. Since restricting to $\mathrm{int}(P)$
cannot create more regions than the global arrangement, we have
\begin{equation}
N_{\mathrm{arr}}(P)\ \le\ \#\bigl(\text{regions of } \mathbb{R}^d\setminus \cup_{r=1}^n H_r\bigr).
\end{equation}

For an arrangement of $n$ affine hyperplanes in $\mathbb{R}^d$, the number of regions
is always upper bounded by $\sum_{k=0}^d \binom{n}{k}$ (with equality under general
position)~\cite{book37}. Therefore,
\begin{equation}
N(P;\mathcal{J})\ \le\ N_{\mathrm{arr}}(P)\ \le\ \sum_{k=0}^d \binom{n}{k}.
\end{equation}

which proves Theorem~\ref{thm:monotone-cells}(iii).

\paragraph{Step 4 (Specialize to $P_\varepsilon(x)$).}
Taking $P=P_\varepsilon(x)$ gives $I_\varepsilon(x;\mathcal{J})=I(P_\varepsilon(x);\mathcal{J})$ and
$N_\varepsilon(x;\mathcal{J})=N(P_\varepsilon(x);\mathcal{J})$, completing the proof.
\end{proof}

\section{Proof of Theorem~\ref{thm:distance-intersection}}
\label{Proof2}
\begin{proof}
We proceed in two steps. First we show that the directional thickness
$\delta_\varepsilon(x;v)$ is strictly positive for every nonzero direction
$v\in\mathbb{R}^d$. We then prove that condition
\eqref{eq:intersection-condition-polytope} implies that the neuron zero set
$Z_{\ell,i}$ intersects $\mathrm{int}\bigl(P_\varepsilon(x)\bigr)$.

\paragraph{Step 1: Positivity of the directional thickness.}
Since $x\in\mathrm{int}\bigl(P_\varepsilon(x)\bigr)$ and
$\mathrm{int}\bigl(P_\varepsilon(x)\bigr)$ is an open subset of $\mathbb{R}^d$,
there exists $\eta>0$ such that
\begin{equation}
\label{eq:ball-in-polytope}
\bigl\{y\in\mathbb{R}^d : \|y-x\| < \eta\bigr\}
\;\subset\;
\mathrm{int}\bigl(P_\varepsilon(x)\bigr).
\end{equation}
Let $v\in\mathbb{R}^d$ with $v\neq 0$ and define the unit vector
\begin{equation}
u_v := \frac{v}{\|v\|}.
\end{equation}
For every $s\in[-\eta,\eta]$ we have
\begin{equation}
\bigl\| x + s\,u_v - x \bigr\|
=
|s|\,\|u_v\|
=
|s|
\;<\;
\eta,
\end{equation}
and hence, by \eqref{eq:ball-in-polytope},
\begin{equation}
x + s\,u_v \in \mathrm{int}\bigl(P_\varepsilon(x)\bigr)
\quad\text{for all } s\in[-\eta,\eta].
\end{equation}
By the definition of the directional thickness,
\begin{equation}
\delta_\varepsilon(x;v)
=
\sup\Bigl\{ t>0 : x + s\,u_v \in \mathrm{int}\bigl(P_\varepsilon(x)\bigr)
\;\text{for all } s\in[-t,t]\Bigr\},
\end{equation}
this implies
\begin{equation}
\delta_\varepsilon(x;v) \;\ge\; \eta \;>\; 0.
\end{equation}
Thus $\delta_\varepsilon(x;v)>0$ for all $v\neq 0$.

\paragraph{Step 2: Implication of \eqref{eq:intersection-condition-polytope}.}
Fix a neuron $(\ell,i)$ and recall that, on $R_{\mathrm{parent}}(x)$, the
pre-activation admits the affine representation
\begin{equation}
z_{\ell,i}(y)
=
a_{\ell,i}^\top y + c_{\ell,i},
\qquad y\in R_{\mathrm{parent}}(x),
\end{equation}
with $a_{\ell,i}\neq 0$ by assumption. The corresponding input-space hyperplane is
\begin{equation}
H_{\ell,i}
:=
\bigl\{y\in\mathbb{R}^d : a_{\ell,i}^\top y + c_{\ell,i} = 0\bigr\}.
\end{equation}
By definition of the neuron zero set,
\begin{equation}
Z_{\ell,i}
=
\bigl\{y\in\mathbb{R}^d : z_{\ell,i}(y)=0\bigr\},
\end{equation}
and on $R_{\mathrm{parent}}(x)$ we therefore have
\begin{equation}
Z_{\ell,i}\cap R_{\mathrm{parent}}(x)
=
H_{\ell,i}\cap R_{\mathrm{parent}}(x).
\end{equation}

We distinguish two cases depending on the value of $z_{\ell,i}(x)$.

\medskip
\noindent\emph{Case 1: $z_{\ell,i}(x)=0$.}
In this case $x\in Z_{\ell,i}$, and since
$x\in\mathrm{int}\bigl(P_\varepsilon(x)\bigr)$ by assumption, we directly obtain
\begin{equation}
x\in Z_{\ell,i}\cap\mathrm{int}\bigl(P_\varepsilon(x)\bigr),
\end{equation}
so $Z_{\ell,i}\cap\mathrm{int}\bigl(P_\varepsilon(x)\bigr)\neq\varnothing$ and the
conclusion holds.

\medskip
\noindent\emph{Case 2: $z_{\ell,i}(x)\neq 0$.}
Define the unit normal
\begin{equation}
u := \frac{a_{\ell,i}}{\|a_{\ell,i}\|},
\end{equation}
and recall the point-to-hyperplane distance formula:
\begin{equation}
\mathrm{dist}\bigl(x,H_{\ell,i}\bigr)
=
\frac{\bigl|a_{\ell,i}^\top x + c_{\ell,i}\bigr|}
     {\|a_{\ell,i}\|}
=
\frac{|z_{\ell,i}(x)|}{\|a_{\ell,i}\|}.
\end{equation}
Set
\begin{equation}
d := \mathrm{dist}\bigl(x,H_{\ell,i}\bigr)
=
\frac{|z_{\ell,i}(x)|}{\|a_{\ell,i}\|}.
\end{equation}
We will construct a point $y^\star\in H_{\ell,i}\cap\mathrm{int}\bigl(P_\varepsilon(x)\bigr)$.

Since $z_{\ell,i}(x)\neq 0$, the sign function
\begin{equation}
\operatorname{sign}\bigl(z_{\ell,i}(x)\bigr)
\in\{-1,+1\}
\end{equation}
is well-defined. Using condition \eqref{eq:intersection-condition-polytope} with
$v=a_{\ell,i}$, we obtain
\begin{equation}
d
=
\frac{|z_{\ell,i}(x)|}{\|a_{\ell,i}\|}
\;<\;
\delta_\varepsilon\bigl(x;a_{\ell,i}\bigr).
\end{equation}
Define
\begin{equation}
\delta := \delta_\varepsilon\bigl(x;a_{\ell,i}\bigr),
\end{equation}
so that $0<d<\delta$. Let
\begin{equation}
S
:=
\Bigl\{\,t>0:\ x+s\,u \in \mathrm{int}\bigl(P_\varepsilon(x)\bigr)
\ \text{for all } s\in[-t,t]\Bigr\},
\end{equation}
where $u:=a_{\ell,i}/\|a_{\ell,i}\|$. By definition of the directional thickness,
we have $\delta=\sup S$. Set
\begin{equation}
\epsilon := \frac{\delta-d}{2} \;>\; 0.
\end{equation}
By the defining property of the supremum, there exists $t_S\in S$ such that
\begin{equation}
t_S \;>\; \delta-\epsilon \;=\; \frac{\delta+d}{2} \;>\; d.
\end{equation}
Consequently,
\begin{equation}
x+s\,u \in \mathrm{int}\bigl(P_\varepsilon(x)\bigr)
\quad\text{for all } s\in[-t_S,t_S].
\end{equation}

Now define
\begin{equation}
y^\star := x - \operatorname{sign}\bigl(z_{\ell,i}(x)\bigr)\,d\,u.
\end{equation}
Since $|d|<t_S$, we have
\begin{equation}
-\operatorname{sign}\bigl(z_{\ell,i}(x)\bigr)\,d \in [-t_S,t_S],
\end{equation}
and hence
\begin{equation}
y^\star = x + s^\star u \in \mathrm{int}\bigl(P_\varepsilon(x)\bigr),
\end{equation}
where
\begin{equation}
s^\star := -\operatorname{sign}\bigl(z_{\ell,i}(x)\bigr)\,d.
\end{equation}

It remains to show that $y^\star\in H_{\ell,i}$. Using
$u=a_{\ell,i}/\|a_{\ell,i}\|$ and $d=|z_{\ell,i}(x)|/\|a_{\ell,i}\|$, we compute
\begin{equation}
\begin{aligned}
a_{\ell,i}^\top y^\star + c_{\ell,i}
&=
a_{\ell,i}^\top\Bigl(x - \operatorname{sign}\bigl(z_{\ell,i}(x)\bigr)\,d\,u\Bigr)
+ c_{\ell,i}
\\
&=
\bigl(a_{\ell,i}^\top x + c_{\ell,i}\bigr)
-
\operatorname{sign}\bigl(z_{\ell,i}(x)\bigr)\,d\,a_{\ell,i}^\top u
\\
&=
z_{\ell,i}(x)
-
\operatorname{sign}\bigl(z_{\ell,i}(x)\bigr)\,
\frac{|z_{\ell,i}(x)|}{\|a_{\ell,i}\|}\,
\|a_{\ell,i}\|
\\
&=
z_{\ell,i}(x)
-
\operatorname{sign}\bigl(z_{\ell,i}(x)\bigr)\,|z_{\ell,i}(x)|
\\
&= 0,
\end{aligned}
\end{equation}
since $z_{\ell,i}(x)
=
\operatorname{sign}\bigl(z_{\ell,i}(x)\bigr)\,|z_{\ell,i}(x)|$ when
$z_{\ell,i}(x)\neq 0$. Thus $y^\star\in H_{\ell,i}$, and combining this with the
previous inclusion yields
\begin{equation}
y^\star \in H_{\ell,i}\cap \mathrm{int}\bigl(P_\varepsilon(x)\bigr)
\subset Z_{\ell,i}\cap \mathrm{int}\bigl(P_\varepsilon(x)\bigr).
\end{equation}
In particular,
\begin{equation}
Z_{\ell,i}\cap\mathrm{int}\bigl(P_\varepsilon(x)\bigr)\neq\varnothing.
\end{equation}

\medskip
In both cases, we have shown that condition
\eqref{eq:intersection-condition-polytope} implies
$Z_{\ell,i}\cap\mathrm{int}\bigl(P_\varepsilon(x)\bigr)\neq\varnothing$, which
completes the proof of Theorem~\ref{thm:distance-intersection}.
\end{proof}

\begin{figure*}[ht]
  \centering
  \begin{subfigure}[b]{0.98\textwidth}
    \centering
    \includegraphics[width=\textwidth]{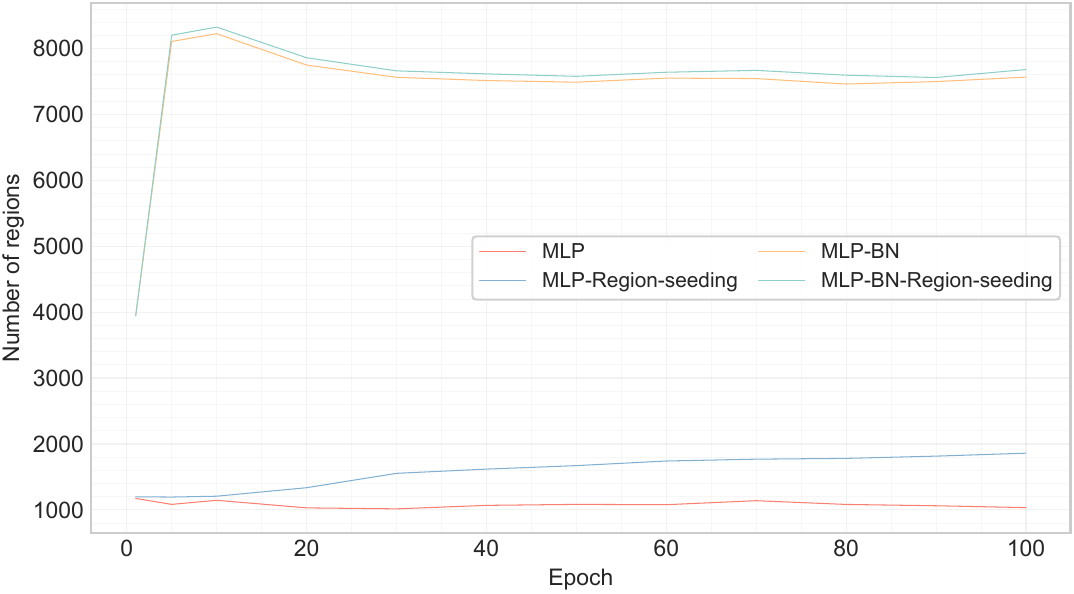}
    \caption{Feedforward MLP}
    \label{fig:moon_regions_vis_mlp}
  \end{subfigure}
  \\[10pt] 
  \begin{subfigure}[b]{0.98\textwidth}
    \centering
    \includegraphics[width=\textwidth]{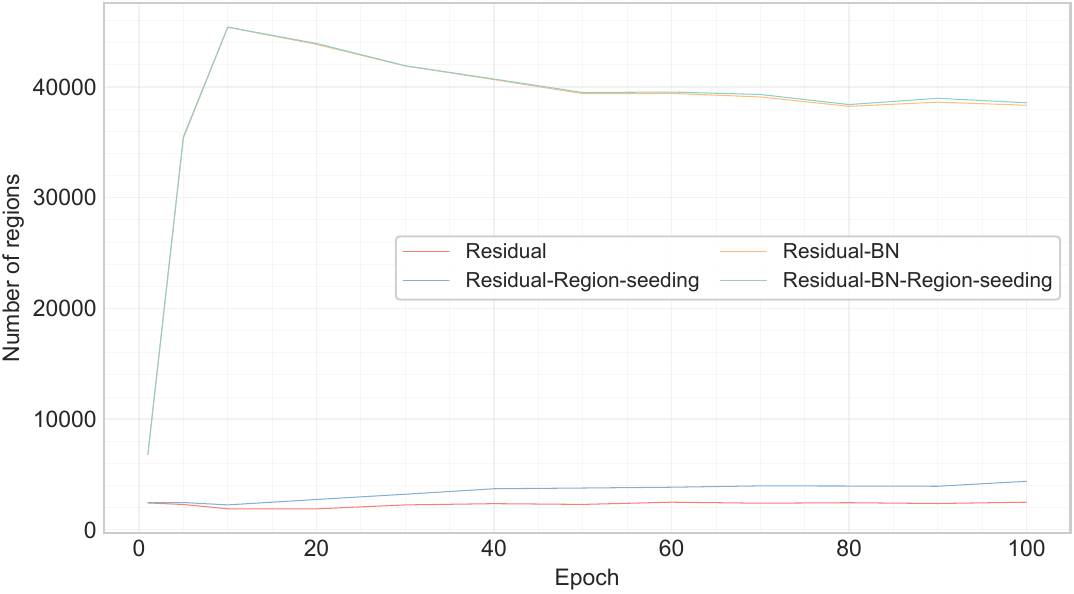}
    \caption{Residual connections}
    \label{fig:moon_regions_vis_residual}
  \end{subfigure}
  \caption{Two Moons dataset: exact affine-region partition visualizations. Comparison of baseline and region-seeded partitions for MLP and ResNet architectures. The region-seeded models demonstrate more intricate partitioning.}
  \label{fig:moon_regions_vis}
\end{figure*}
\begin{figure*}[ht]
  \centering
  \begin{subfigure}[t]{0.98\textwidth}
    \centering
    \includegraphics[width=\textwidth]{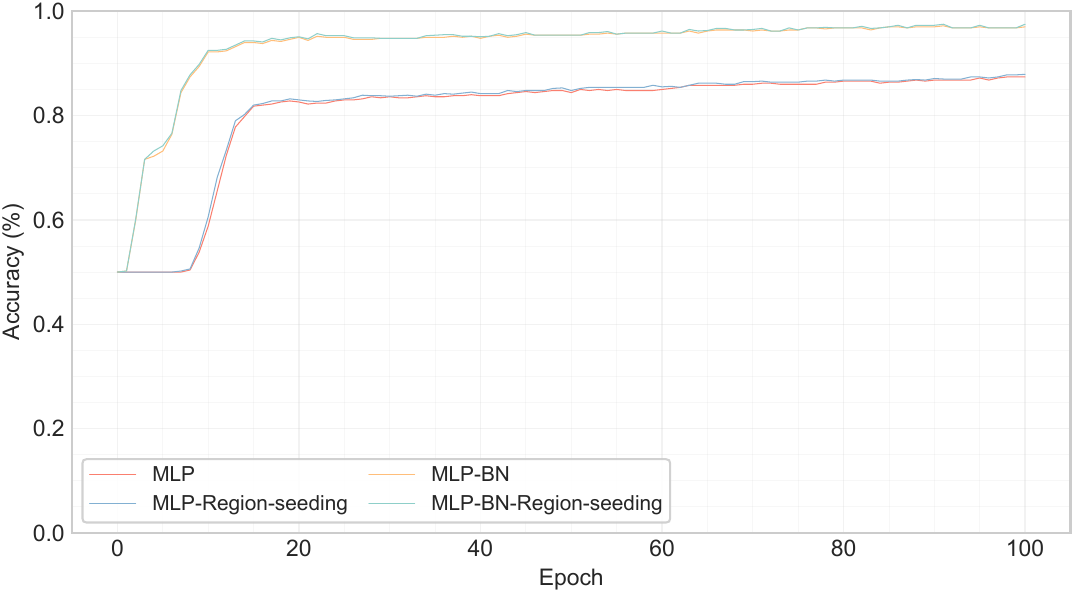}
    \caption{Feedforward MLP}
    \label{fig:moon_acc_mlp}
  \end{subfigure}
  \hfill
  \begin{subfigure}[t]{0.98\textwidth}
    \centering
    \includegraphics[width=\textwidth]{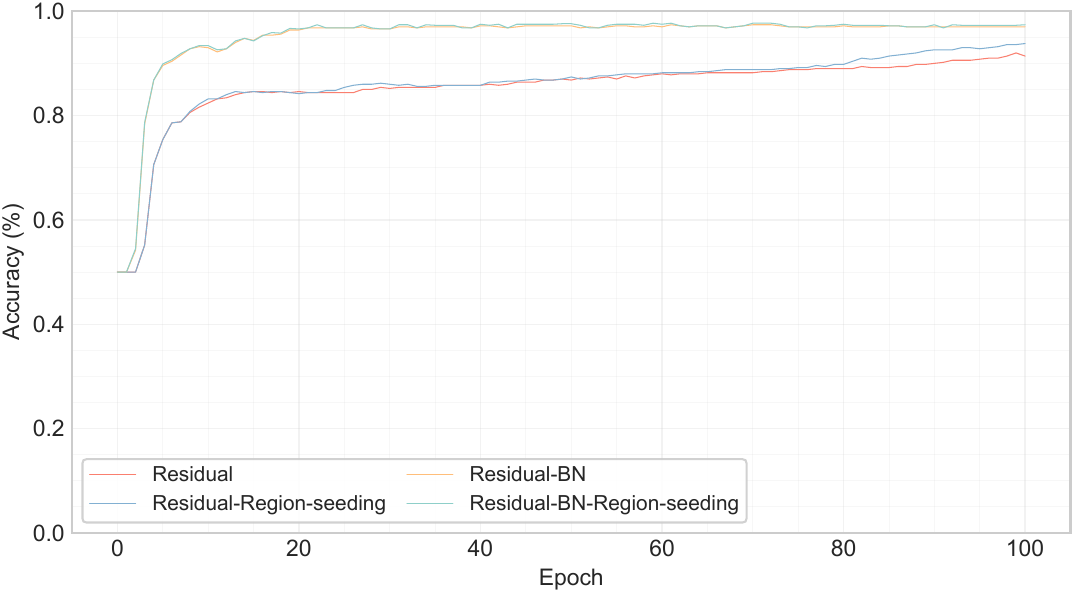}
    \caption{Residual connections}
    \label{fig:moon_acc_res}
  \end{subfigure}
  \caption{Two Moons dataset: test accuracy over training epochs. The regularized models show improved performance.}
  \label{fig:moon_acc}
\end{figure*}

\begin{figure*}[ht]
  \centering
  \begin{subfigure}[t]{0.98\textwidth}
    \centering
    \includegraphics[width=\textwidth]{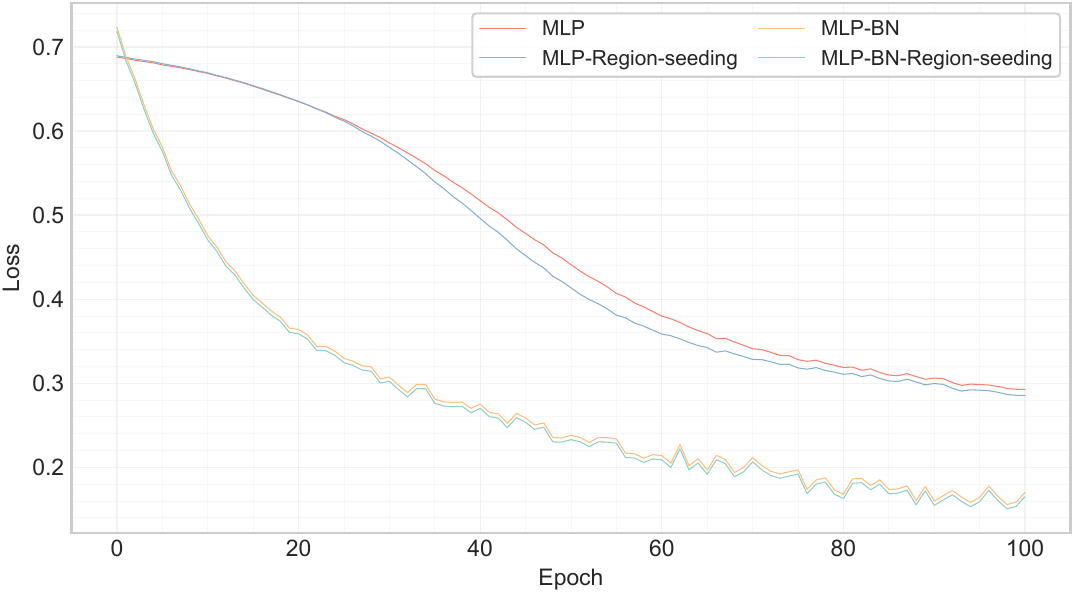}
    \caption{Feedforward MLP}
    \label{fig:moon_loss_mlp}
  \end{subfigure}
  \hfill
  \begin{subfigure}[t]{0.98\textwidth}
    \centering
    \includegraphics[width=\textwidth]{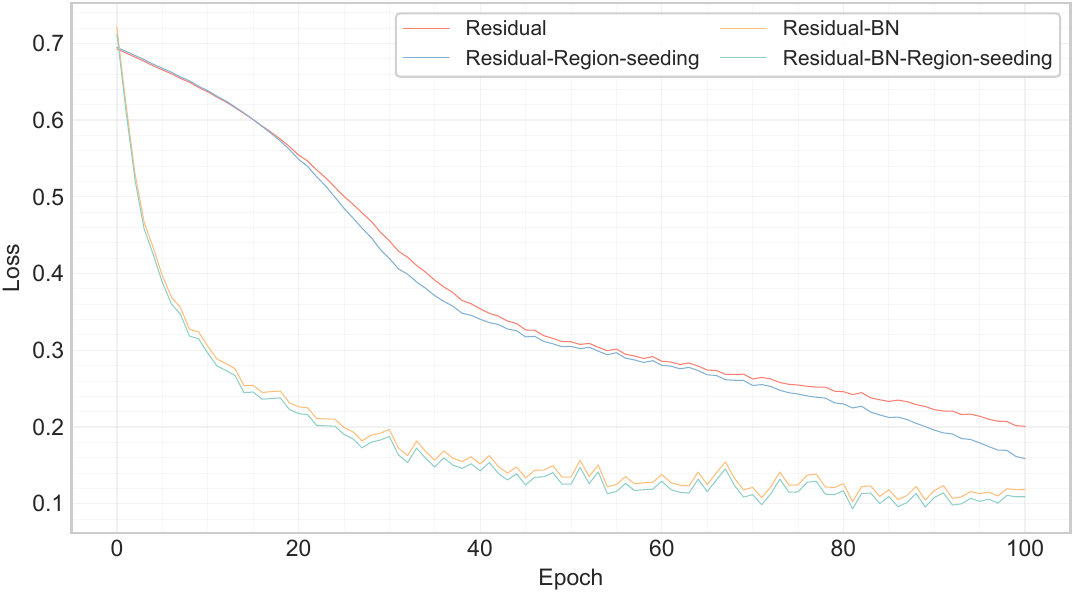}
    \caption{Residual connections}
    \label{fig:moon_loss_res}
  \end{subfigure}
  \caption{Two Moons dataset: optimization dynamics (task loss) over training epochs.}
  \label{fig:moon_loss}
\end{figure*}

\begin{figure*}[ht]
  \centering
  \begin{subfigure}[t]{0.98\textwidth}
    \centering
    \includegraphics[width=\textwidth]{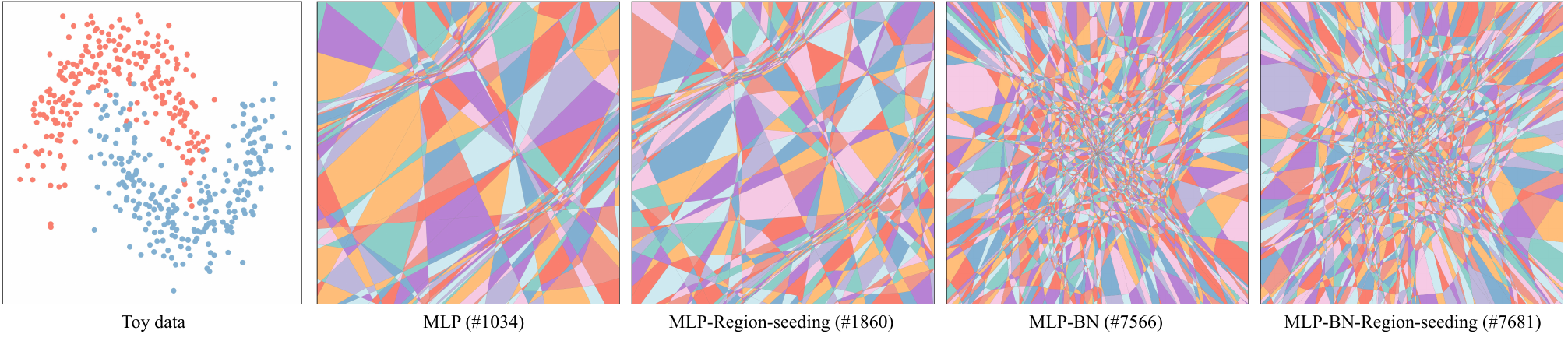}
    \caption{Feedforward MLP}
    \label{fig:moon_regions_mlp}
  \end{subfigure}
  \hfill
  \begin{subfigure}[t]{0.98\textwidth}
    \centering
    \includegraphics[width=\textwidth]{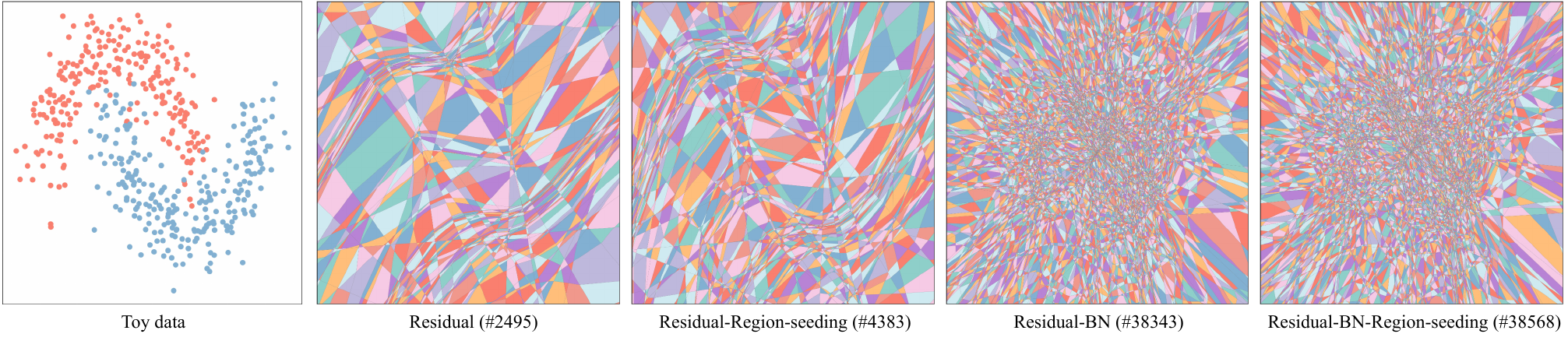}
    \caption{Residual connections}
    \label{fig:moon_regions_res}
  \end{subfigure}
  \caption{Two Moons dataset: exact realized affine-region counts.}
  \label{fig:moon_regions}
\end{figure*}

\section{Toy-data Experiments: Implementation Details and Additional Results}
\label{C}

In this section, we provide the exact implementation details for the synthetic data experiments reported in the main text. Furthermore, we present additional experimental results on the Two Moons and Gaussian Quantiles datasets to demonstrate the robustness of the proposed region-seeding regularizer across different data geometries and network capacities.

\subsection{Hyperparameter Settings for Main Text Experiments}
\label{app:toy_hyperparams}

We report the configuration used for the random dataset experiments discussed in Section~\ref{subsec:toy}. To ensure robustness, all experiments were repeated across three different random seeds, and we report the average performance.

\paragraph{Optimization and Training.}
All models were trained for $1,000$ epochs using the Adam optimizer. The specific optimization hyperparameters are listed in Table~\ref{tab:toy_optim_params}.

\begin{table}[ht]
\centering
\caption{Optimization hyperparameters for toy dataset experiments.}
\label{tab:toy_optim_params}
\begin{small}
\begin{sc}
\begin{tabular}{lc}
\toprule
Parameter & Value \\
\midrule
Optimizer & Adam \\
Learning Rate & $1 \times 10^{-4}$ \\
Betas & $(0.9, 0.999)$ \\
Weight Decay & $0$ \\
Batch Size & $64$ \\
Max Epochs & $1000$ \\
\bottomrule
\end{tabular}
\end{sc}
\end{small}
\end{table}

\paragraph{Dataset Specifications.}
The primary synthetic dataset ("Random") consists of $N=500$ samples in $\mathbb{R}^2$. The targets are assigned to $2$ distinct classes.

\paragraph{Model Architectures.}
We evaluated two distinct architectures: a standard feedforward MLP and a Residual Network (ResNet). Both architectures process 2-dimensional inputs and produce logits for 2 classes.
\begin{itemize}
    \item \textbf{Feedforward MLP:} A 5-layer fully connected network. The architecture follows the structure: Input $\to$ [Linear(32) $\to$ Norm $\to$ ReLU] $\times 5$ $\to$ Linear(2). The hidden width is uniform at $32$ neurons.
    \item \textbf{ResNet-MLP:} A residual architecture consisting of an input projection followed by 3 residual blocks.
    \begin{itemize}
        \item \emph{Input Projection:} Linear($2 \to 32$) $\to$ Norm $\to$ ReLU.
        \item \emph{Residual Block:} $x_{out} = \mathrm{ReLU}(x + \mathcal{F}(x))$, where $\mathcal{F}$ consists of [Linear($32 \to 32$) $\to$ Norm $\to$ ReLU $\to$ Linear($32 \to 32$) $\to$ Norm].
        \item \emph{Output:} Linear($32 \to 2$).
    \end{itemize}
\end{itemize}
For normalization, we compared two settings: \emph{Identity} (no normalization) and \emph{Batch Normalization} (standard 1D BatchNorm with affine parameters).

\subsection{Region-Seeding Regularization Implementation}
\label{app:region_seeding_config}

The region-seeding regularizer is implemented as a plug-and-play wrapper module. Based on the provided source code, the specific functional forms for the depth-dependent weights $\lambda_\ell$ and the annealing schedule $\eta(t)$ used in Equation~\eqref{eq:rtotal} are defined as follows.

\paragraph{Pre-activation Penalty Definition.}
Consistent with the main text, the penalty for a specific layer $\ell$ is calculated as the $\ell_2$ norm of the pre-activation vector normalized by the number of neurons $n_\ell$. For a batch $B$, this is given by:
\begin{equation}
\mathcal{R}_\ell(\theta; B) = \frac{1}{|B|} \sum_{b=1}^{|B|} \frac{\|u_\ell(x_b)\|_2}{n_\ell},
\end{equation}
where \(|B|\) is the batch size. The implementation applies this penalty to all activation modules in the network (specifically ReLU for CPA networks, though the code supports LeakyReLU and GELU).

\paragraph{Depth-Dependent Layer Weighting ($\lambda_\ell$).}
The implementation uses a ``decay'' mode for layer weighting, designed to assign higher regularization strength to earlier layers. For a network with a total of $L$ regularized activation layers, the weight $\lambda_\ell$ for the $\ell$-th layer (1-indexed, $\ell \in \{1, \dots, L\}$) is computed as:
\begin{equation}
\label{eq:layer_decay}
\lambda_\ell = 4 \cdot \frac{L - \ell + 1}{L + 1}.
\end{equation}
This formula ensures that weights decay linearly with depth. For example, in the 5-layer MLP ($L=5$), the weights range from $\lambda_1 \approx 3.33$ to $\lambda_5 \approx 0.67$.

\paragraph{Annealing Schedule ($\eta(t)$).}
The regularization strength is annealed linearly over the training process to allow task-driven refinement in later epochs. Given the current epoch $t$ and the maximum training epochs $T_{\max}=1000$, the annealing factor is:
\begin{equation}
\eta(t) = \max\left(0, 1 - \frac{t}{T_{\max}}\right).
\end{equation}

\paragraph{Total Regularization Term.}
Combining these components, the final regularization term added to the loss is:
\begin{equation}
\mathcal{R}(\theta; B, t) = \alpha \cdot \eta(t) \cdot \sum_{\ell=1}^{L} \lambda_\ell \mathcal{R}_\ell(\theta; B).
\end{equation}
For all toy experiments utilizing the regularizer, we set the global scaling factor $\alpha = 1 \times 10^{-2}$.

\subsection{Geometric Analysis of Neuron Switching Hyperplanes}
\label{app:geometric_analysis}

To empirically validate the mechanism of the proposed regularizer, we analyze the geometric distribution of neuron switching hyperplanes relative to the data. Specifically, we investigate the Euclidean distances between input samples $x$ and the switching surfaces induced by the network's neurons.

\paragraph{Geometric Definitions.}
As derived in Eq.~\ref{eq:dist}, the Euclidean distance from a data point $x$ to the switching hyperplane $H_{\ell,i}$ of neuron $(\ell, i)$ within its local region is given by:
\begin{equation}
    d_{\ell,i}(x) \;=\; \mathrm{dist}\bigl(x,H_{\ell,i}\bigr) \;=\; \frac{|z_{\ell,i}(x)|}{\|a_{\ell,i}\|_2},
\end{equation}
where $z_{\ell,i}(x)$ is the pre-activation value and $a_{\ell,i}$ is the input-space normal vector defined in Eq.~\eqref{eq:a-c-definition}. This metric quantifies the geometric margin between the data and the neuron's linear boundary.

\paragraph{Visualization Methodology.}
In Figure~\ref{fig:log_dist_hist}, we plot the histograms of these distances for the random (Figure~\ref{fig:toy_regions_vis}) dataset. The visualization details are as follows:
\begin{itemize}
    \item \textbf{Metric (X-axis):} We plot the log-distance, $\log(d_{\ell,i}(x))$. A value of $\log(d) < 0$ implies $d < 1$ (relative to input scale), indicating that the data point is proximal to the hyperplane. Smaller (more negative) values correspond to tighter alignment between the data and the switching surfaces.
    \item \textbf{Count (Y-axis):} The number of neuron-sample pairs falling into each distance bin.
    \item \textbf{Scope (Legend):}
    \begin{itemize}
        \item \textbf{All Neurons:} Represents the total potential switching surfaces for the hidden layer. With $500$ samples and a layer width of $32$, the total number of distance measurements is $500 \times 32 = 16,000$.
        \item \textbf{Intersect Neurons:} Aggregates distances only for the subset of neurons whose switching surfaces effectively intersect or bound the local polytope containing $x$. This represents the active local geometry.
    \end{itemize}
\end{itemize}

\begin{figure}[ht]
  \centering
  \begin{subfigure}[b]{0.98\linewidth}
    \centering
    \includegraphics[width=\linewidth]{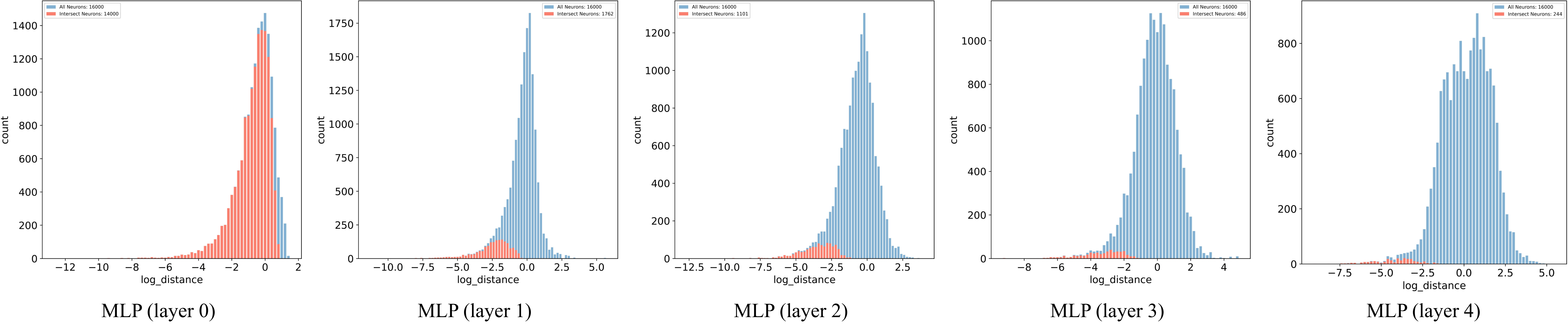} 
    \caption{MLP Model. Distribution of log-distances.}
    \label{fig:hist_baseline}
  \end{subfigure}
  \begin{subfigure}[b]{0.98\linewidth}
    \centering
    \includegraphics[width=\linewidth]{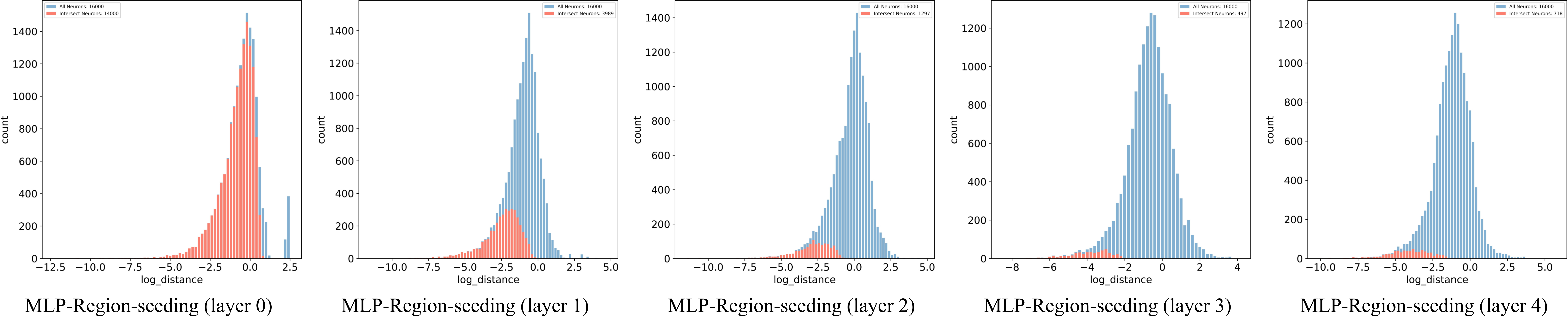}
    \caption{MLP-Region-seeding Model. Distribution of log-distances.}
    \label{fig:hist_seeded}
  \end{subfigure}
  \caption{Distribution of Data-to-Hyperplane Distances. Histograms of $\log(d(x, H))$ for the MLP trained on the random toy dataset (Figure~\ref{fig:toy_regions_vis}). The MLP-Region-seeding Model (bottom) exhibits a distribution significantly shifted towards negative values compared to the Baseline (top). This indicates that the regularizer successfully encourages neurons to form switching surfaces in close proximity to the data, thereby "seeding" a finer partition of the input space during the optimization.}
  \label{fig:log_dist_hist}
\end{figure}

\paragraph{Analysis of Results.}
Comparing the distributions in Figure~\ref{fig:log_dist_hist}, we observe a distinct shift in the geometry of the Region-Seeded model.
\begin{enumerate}
    \item \textbf{Distribution Shift:} The histogram for the Region-Seeded model (Figure~\ref{fig:hist_seeded}) is densely concentrated in the negative log-distance region (left side), whereas the Baseline model (Figure~\ref{fig:hist_baseline}) exhibits a broader distribution skewed towards larger positive values.
    \item \textbf{Geometric Interpretation:} The shift towards smaller $\log(d)$ confirms that our regularizer effectively reduces the margin between data points and neuron switching hyperplanes. By pulling these surfaces closer to the data, the model increases the probability of them intersecting the local neighborhood of data points.
    \item \textbf{Conclusion:} This empirical evidence supports our theoretical motivation: minimizing pre-activation magnitudes drives switching hyperplanes to intersect data neighborhoods, thereby increasing the density of realized affine regions and enhancing local expressivity.
\end{enumerate}

\subsection{Additional Results: Two Moons Dataset}
\label{app:toy_moon}

To evaluate the robustness of our approach on a dataset with non-linear, non-convex decision boundaries, we replicated the full experimental suite on the ``Two Moons'' dataset. The experimental setup (width $32$) and hyperparameters mirror the main text exactly, with $\alpha=1 \times 10^{-2}$.

\paragraph{Exact Affine-Region Visualization.}
Figure~\ref{fig:moon_regions} visualizes the input-space partitions. Consistent with the random dataset results, the region-seeded models (both MLP and ResNet) exhibit a visibly finer fragmentation of the input space.

\paragraph{Quantitative Analysis.}
The quantitative metrics confirm the visual intuition.
\begin{itemize}
    \item \textbf{Accuracy (Figure~\ref{fig:moon_acc}):} The region-seeded models achieve higher accuracy in the early epochs compared to their baseline counterparts across both architectures and normalization settings.
    \item \textbf{Optimization Dynamics (Figure~\ref{fig:moon_loss}):} The task loss decreases more rapidly for the region-seeded variants, indicating that the early formation of local regions facilitates faster optimization.
    \item \textbf{Region Counts (Figure~\ref{fig:moon_regions_vis}):} Exact enumeration shows a strictly higher number of realized affine regions for the regularized models throughout the training process.
\end{itemize}

\begin{figure*}[ht]
  \centering
  \begin{subfigure}[b]{0.98\textwidth}
    \centering
    \includegraphics[width=\textwidth]{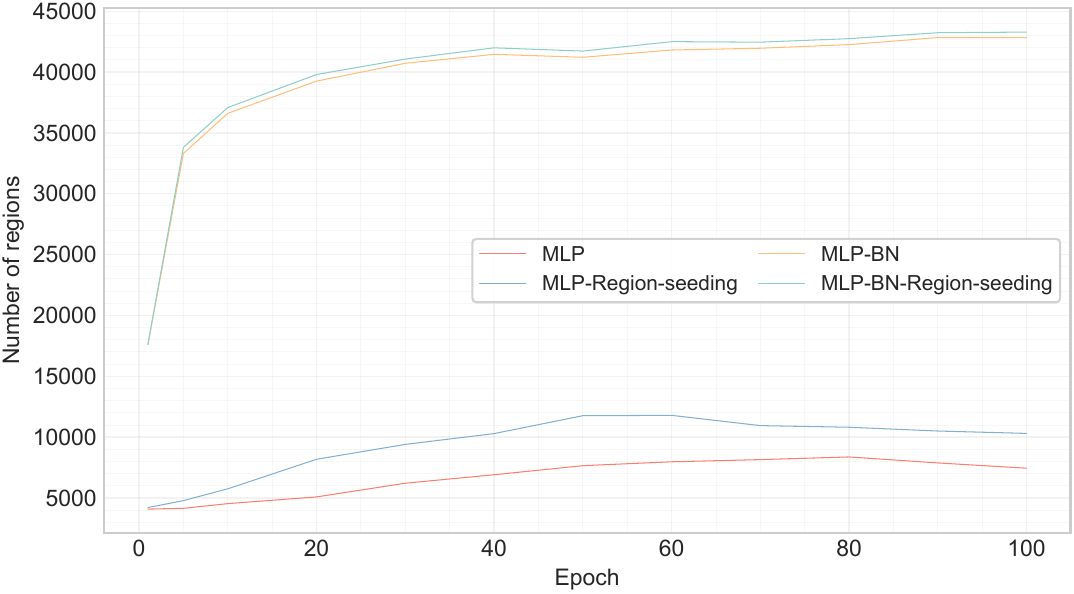}
    \caption{Feedforward MLP (Width 64)}
    \label{fig:gaussian_regions_vis_mlp}
  \end{subfigure}
  \\[10pt] 
  \begin{subfigure}[b]{0.98\textwidth}
    \centering
    \includegraphics[width=\textwidth]{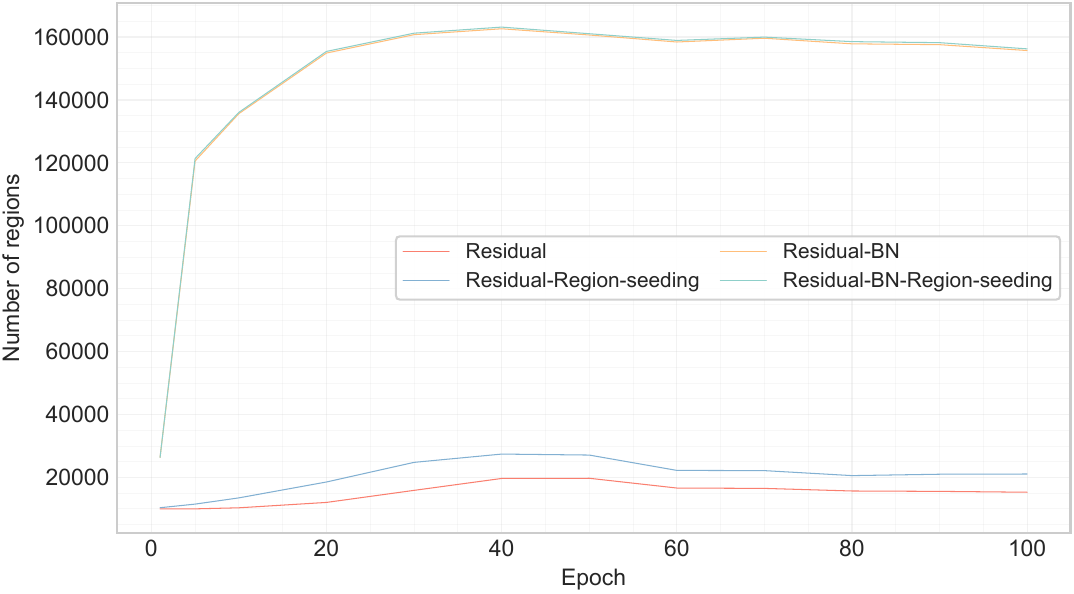}
    \caption{Residual connections (Width 64)}
    \label{fig:gaussian_regions_vis_residual}
  \end{subfigure}
  \caption{Gaussian Quantiles dataset: exact affine-region partition visualizations. The region-seeded models induce a finer partition with more realized affine regions.}
  \label{fig:gaussian_regions_vis}
\end{figure*}

\begin{figure*}[ht]
  \centering
  \begin{subfigure}[t]{0.98\textwidth}
    \centering
    \includegraphics[width=\textwidth]{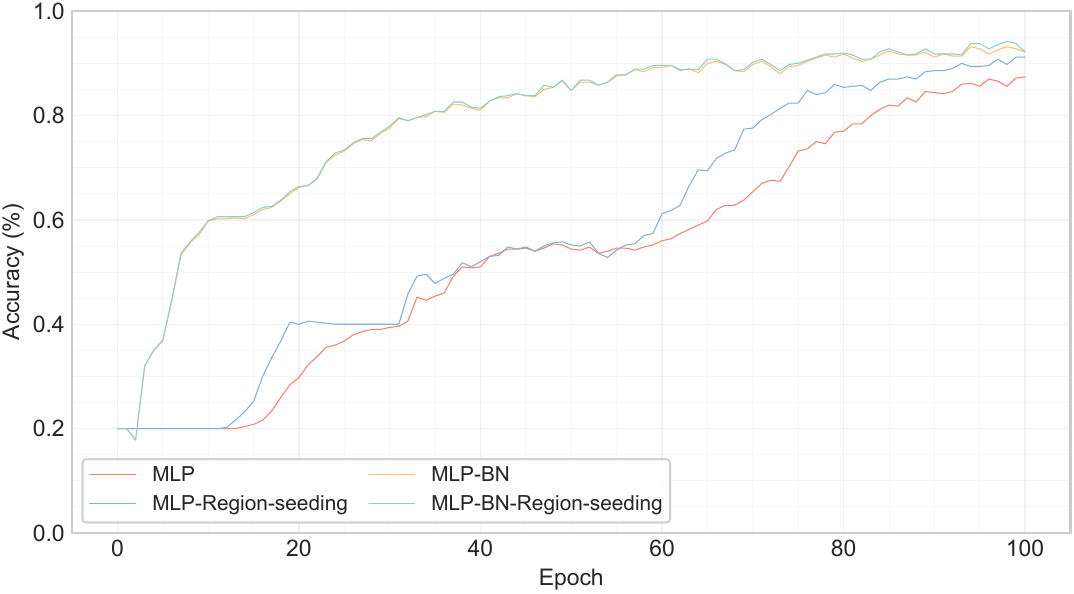}
    \caption{Feedforward MLP}
    \label{fig:gaussian_acc_mlp}
  \end{subfigure}
  \hfill
  \begin{subfigure}[t]{0.98\textwidth}
    \centering
    \includegraphics[width=\textwidth]{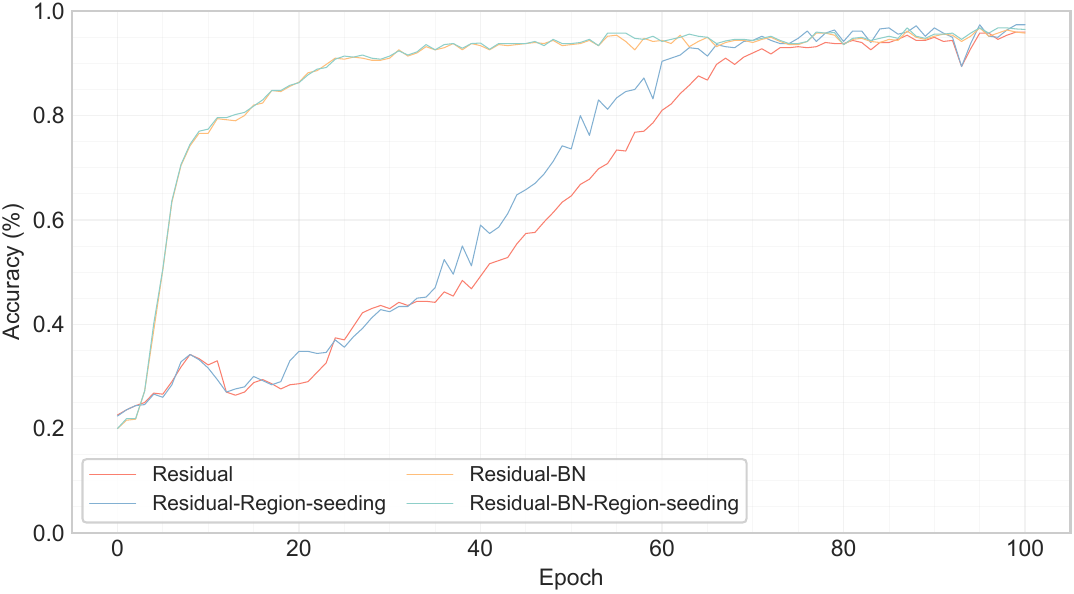}
    \caption{Residual connections}
    \label{fig:gaussian_acc_res}
  \end{subfigure}
  \caption{Gaussian Quantiles dataset: test accuracy over training epochs.}
  \label{fig:gaussian_acc}
\end{figure*}

\begin{figure*}[ht]
  \centering
  \begin{subfigure}[t]{0.98\textwidth}
    \centering
    \includegraphics[width=\textwidth]{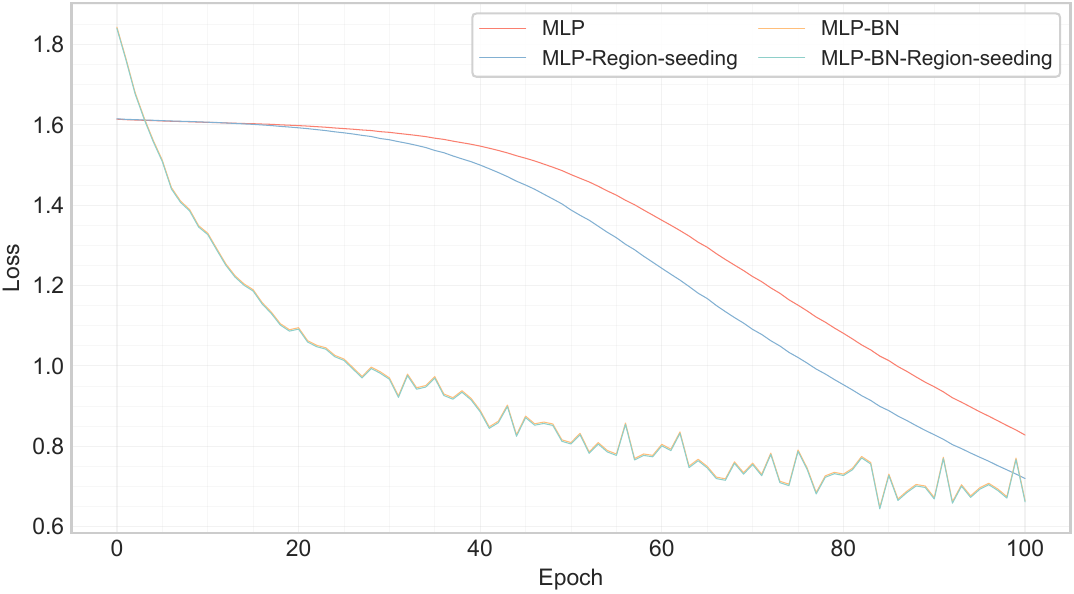}
    \caption{Feedforward MLP}
    \label{fig:gaussian_loss_mlp}
  \end{subfigure}
  \hfill
  \begin{subfigure}[t]{0.98\textwidth}
    \centering
    \includegraphics[width=\textwidth]{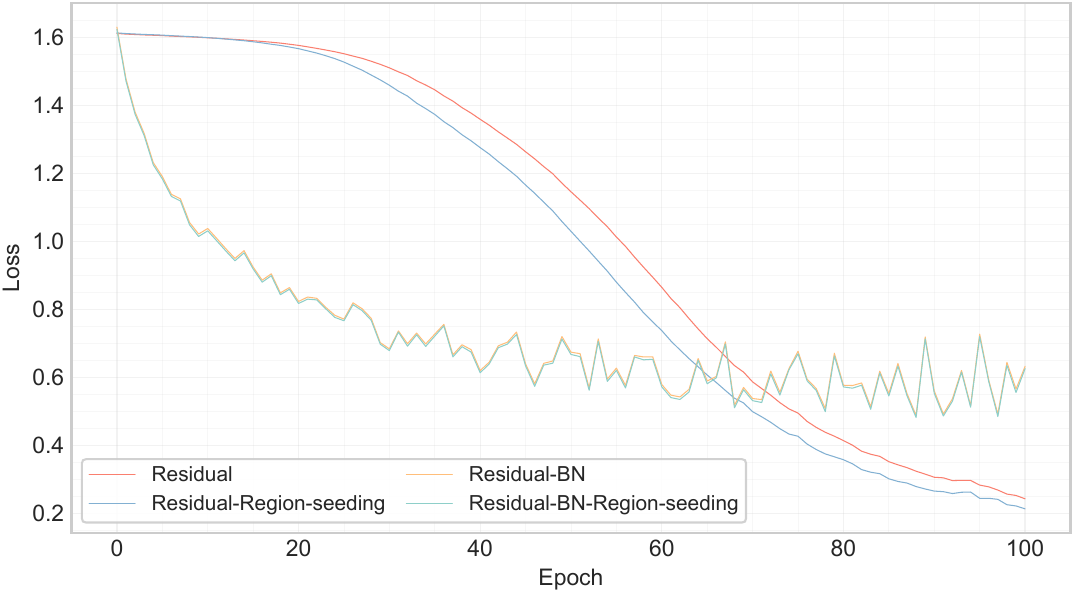}
    \caption{Residual connections}
    \label{fig:gaussian_loss_res}
  \end{subfigure}
  \caption{Gaussian Quantiles dataset: optimization dynamics (task loss) over training epochs.}
  \label{fig:gaussian_loss}
\end{figure*}

\begin{figure*}[ht]
  \centering
  \begin{subfigure}[t]{0.98\textwidth}
    \centering
    \includegraphics[width=\textwidth]{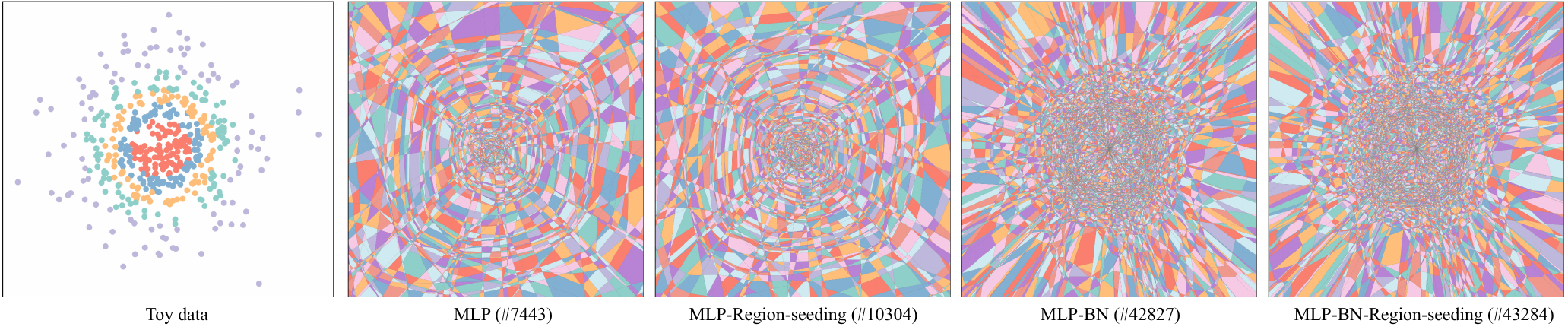}
    \caption{Feedforward MLP}
    \label{fig:gaussian_regions_mlp}
  \end{subfigure}
  \hfill
  \begin{subfigure}[t]{0.98\textwidth}
    \centering
    \includegraphics[width=\textwidth]{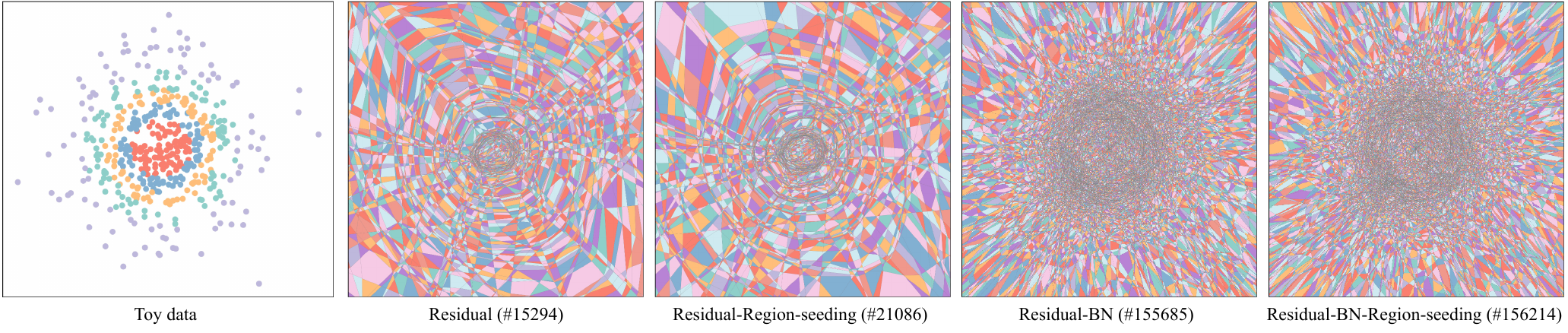}
    \caption{Residual connections}
    \label{fig:gaussian_regions_res}
  \end{subfigure}
  \caption{Gaussian Quantiles dataset: exact realized affine-region counts.}
  \label{fig:gaussian_regions}
\end{figure*}

\subsection{Additional Results: Gaussian Quantiles Dataset}
\label{app:toy_gaussian}

We extended the evaluation to the ``Gaussian Quantiles'' dataset (5 classes), which presents a multi-class problem with concentric ring-like decision structures. For this experiment, the network width was increased to $m=64$ (for both MLP and ResNet) to accommodate the higher complexity, while all other settings remained identical to the main text.

\paragraph{Results.}
Despite the change in dataset geometry and increased model capacity, the conclusions remain consistent with the main text. Figure~\ref{fig:gaussian_regions} illustrates that region seeding successfully promotes partition density adapted to the concentric structure. Figures~\ref{fig:gaussian_acc}, \ref{fig:gaussian_loss}, and \ref{fig:gaussian_regions_vis} demonstrate that the regularizer continues to provide a clear advantage in convergence speed, accuracy, and the number of realized affine regions, proving the method's effectiveness scales to wider networks and multi-class settings.

\section{Real-data Experiments: Implementation Details and  and Ablation Studies}
\label{D}

In this section, we provide the detailed hyperparameter configurations for the ImageNet-1k experiments reported in Section~\ref{sec:imagenet} and present an ablation study analyzing the components of the region-seeding regularizer.

\subsection{Hyperparameter Configurations}
We conducted experiments on ImageNet-1k using standard training recipes adapted for each architecture. All models were trained using the PyTorch framework.

\paragraph{CNN Architectures (VGG-19, ResNet-18, ResNet-50).}
For Convolutional Neural Networks, we followed standard SGD training protocols. The models were trained for 90 epochs using Nesterov Momentum. The region-seeding regularization coefficient $\alpha$ was set to $10^{-3}$.
Detailed hyperparameters are listed in Table~\ref{tab:cnn_hparams}.

\begin{table}[ht]
\caption{Hyperparameter settings for VGG-19, ResNet-18, and ResNet-50 on ImageNet-1k.}
\label{tab:cnn_hparams}
\centering
\begin{small}
\begin{tabular}{lccc}
\toprule
\textbf{Hyperparameter} & \textbf{VGG-19 (noBN)} & \textbf{ResNet-18} & \textbf{ResNet-50} \\
\midrule
Epochs & 90 & 90 & 90 \\
Optimizer & SGD & SGD & SGD \\
Batch Size & 256 & 256 & 256 \\
Base LR & 0.01 & 0.1 & 0.1 \\
LR Schedule & Step (30, 60, 80) & Step (30, 60, 80) & Step (30, 60, 80) \\
Momentum & 0.9 & 0.9 & 0.9 \\
Weight Decay & $5 \times 10^{-4}$ & $1 \times 10^{-4}$ & $1 \times 10^{-4}$ \\
\midrule
\multicolumn{4}{l}{\textit{Region Seeding Regularizer}} \\
Coefficient $\alpha$ & $10^{-3}$ & $10^{-3}$ & $10^{-3}$ \\
Layer Weights $\lambda_\ell$ & Decay & Decay & Decay \\
Annealing $\eta(t)$ & Linear ($90 \to 0$) & Linear ($90 \to 0$) & Linear ($90 \to 0$) \\
\bottomrule
\end{tabular}
\end{small}
\end{table}

\paragraph{Vision Transformer (ViT-B/16).}
For ViT-B/16, we employed a stronger training recipe consistent with modern practices, including AdamW optimizer, Cosine LR schedule with warmup, and strong data augmentation (Mixup, Cutmix, RandAugment). The regularization coefficient $\alpha$ was set to $10^{-2}$, reflecting the different scaling of pre-activations in Transformer blocks compared to CNNs. Details are provided in Table~\ref{tab:vit_hparams}.

\begin{table}[ht]
\caption{Hyperparameter settings for ViT-B/16 on ImageNet-1k.}
\label{tab:vit_hparams}
\centering
\begin{small}
\begin{tabular}{lc}
\toprule
\textbf{Hyperparameter} & \textbf{ViT-B/16} \\
\midrule
Epochs & 300 \\
Global Batch Size & 4096 \\
Optimizer & AdamW \\
Base LR & $3 \times 10^{-3}$ \\
LR Schedule & Cosine + Linear Warmup (30 ep) \\
Weight Decay & 0.3 \\
Label Smoothing & 0.11 \\
Mixup $\alpha$ / Cutmix $\alpha$ & 0.2 / 1.0 \\
AutoAugment & RandAugment \\
\midrule
\multicolumn{2}{l}{\textit{Region Seeding Regularizer}} \\
Coefficient $\alpha$ & $10^{-2}$ \\
Layer Weights $\lambda_\ell$ & Decay \\
Annealing $\eta(t)$ & Linear ($300 \to 0$) \\
\bottomrule
\end{tabular}
\end{small}
\end{table}

\subsection{Ablation Study: Annealing and Layer Weighting}
To isolate the contributions of the temporal annealing strategy and the spatial layer-wise weighting scheme, we conducted an ablation study across four configurations:
\begin{enumerate}
    \item \textbf{No Annealing + Uniform Weights:} The regularization strength is constant throughout training ($\eta(t)=1$) and applied equally to all layers ($\lambda_\ell=1$).
    \item \textbf{No Annealing + Layer Decay:} The regularization strength is constant over time, but weights decrease with depth according to Eq.~\eqref{eq:layer_decay}.
    \item \textbf{Annealing + Uniform Weights:} The regularization decays linearly over epochs, but spatial weights are uniform.
    \item \textbf{Annealing + Layer Decay (Ours):} Both temporal annealing and depth-dependent weighting are applied.
\end{enumerate}

\paragraph{Results and Discussion.}
We assessed these variants on ImageNet-1k training. The visual comparisons of validation accuracy and validation loss for each model are presented in Figure~\ref{fig:ablation_vgg} (VGG-19), Figure~\ref{fig:ablation_res18} (ResNet-18), Figure~\ref{fig:ablation_res50} (ResNet-50), and Figure~\ref{fig:ablation_vit} (ViT-B/16). The comparative results confirm that the proposed configuration (\textbf{Annealing + Layer Decay}) yields the best trade-off between early-stage region formation and final task performance.

\begin{figure}[ht]
  \centering
  \begin{subfigure}[b]{0.98\textwidth}
    \centering
    \includegraphics[width=\linewidth]{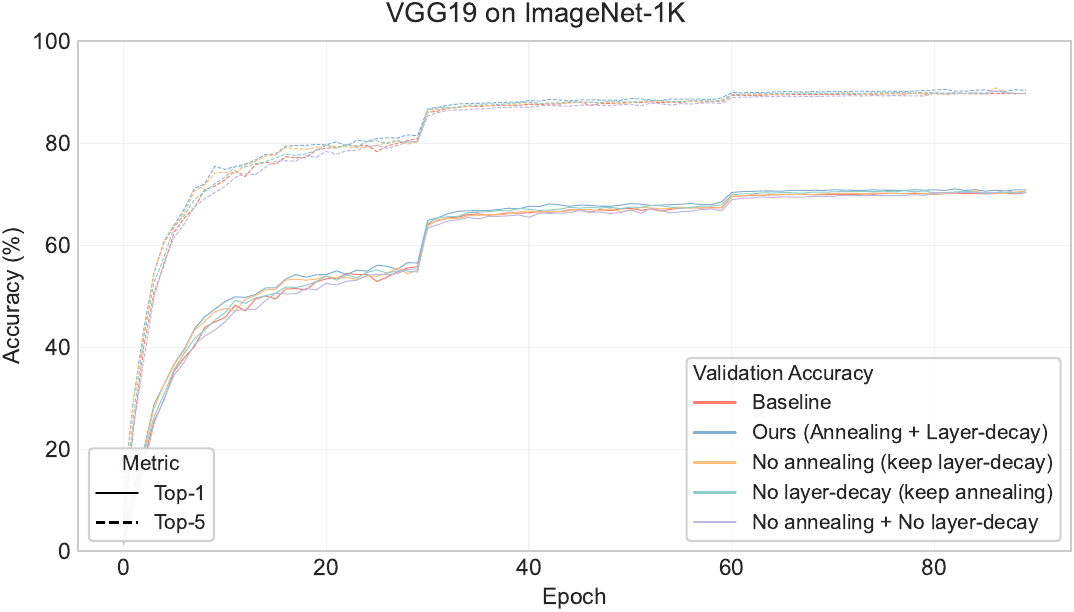}
    \caption{Validation Accuracy}
    \label{fig:ablation_vgg_acc}
  \end{subfigure}
  \hfill
  \begin{subfigure}[b]{0.98\textwidth}
    \centering
    \includegraphics[width=\linewidth]{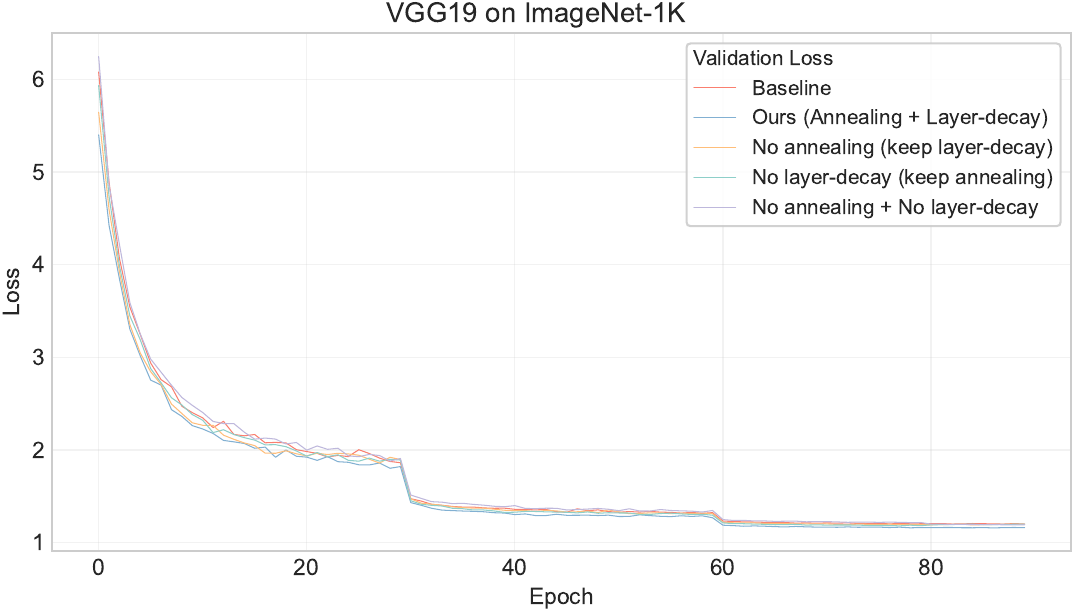}
    \caption{Validation Loss}
    \label{fig:ablation_vgg_loss}
  \end{subfigure}
  \caption{\textbf{Ablation on VGG-19 (noBN).} Comparison of different regularization strategies. The proposed method (Annealing + Layer Decay) achieves superior early convergence and final accuracy compared to non-annealed or uniform-weight baselines.}
  \label{fig:ablation_vgg}
\end{figure}

\begin{figure}[ht]
  \centering
  \begin{subfigure}[b]{0.98\textwidth}
    \centering
    \includegraphics[width=\linewidth]{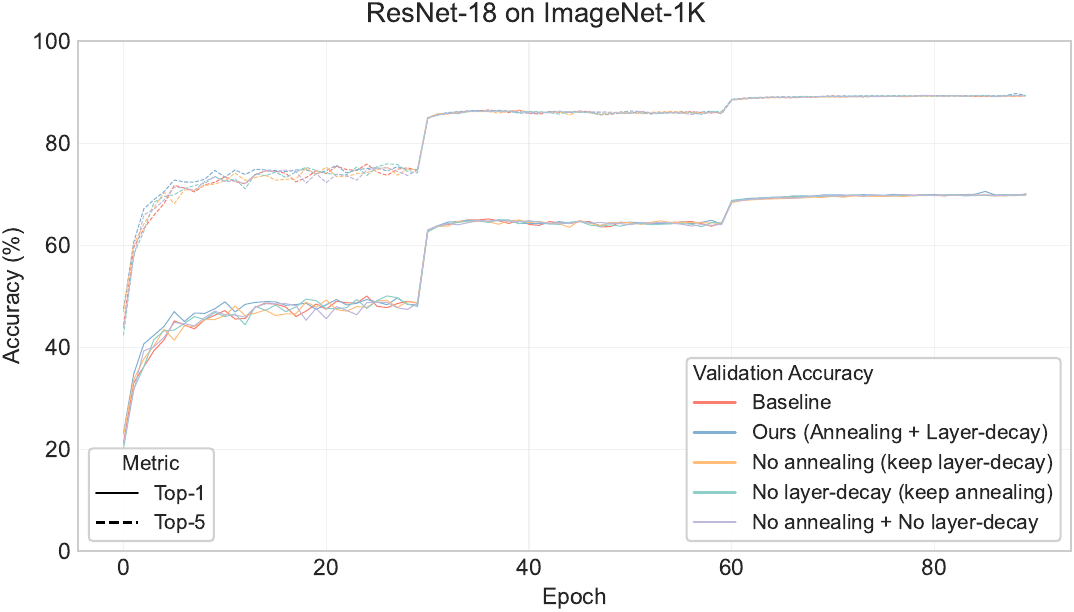}
    \caption{Validation Accuracy}
    \label{fig:ablation_res18_acc}
  \end{subfigure}
  \hfill
  \begin{subfigure}[b]{0.98\textwidth}
    \centering
    \includegraphics[width=\linewidth]{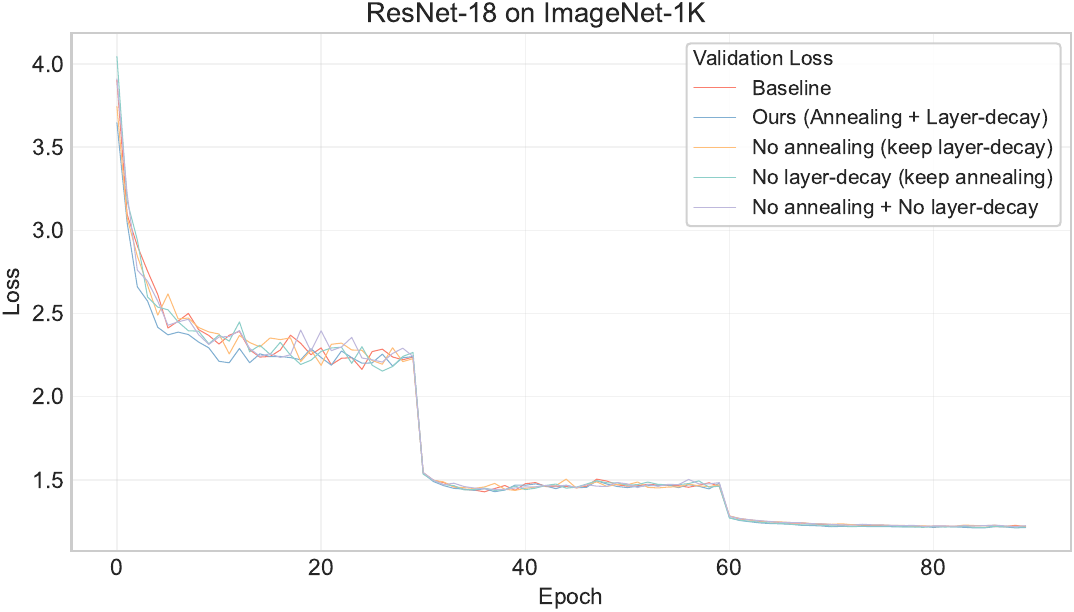}
    \caption{Validation Loss}
    \label{fig:ablation_res18_loss}
  \end{subfigure}
  \caption{\textbf{Ablation on ResNet-18.} Validation accuracy and validation loss trajectories. The combination of temporal annealing and spatial layer decay provides the most robust optimization profile.}
  \label{fig:ablation_res18}
\end{figure}

\begin{figure}[ht]
  \centering
  \begin{subfigure}[b]{0.98\textwidth}
    \centering
    \includegraphics[width=\linewidth]{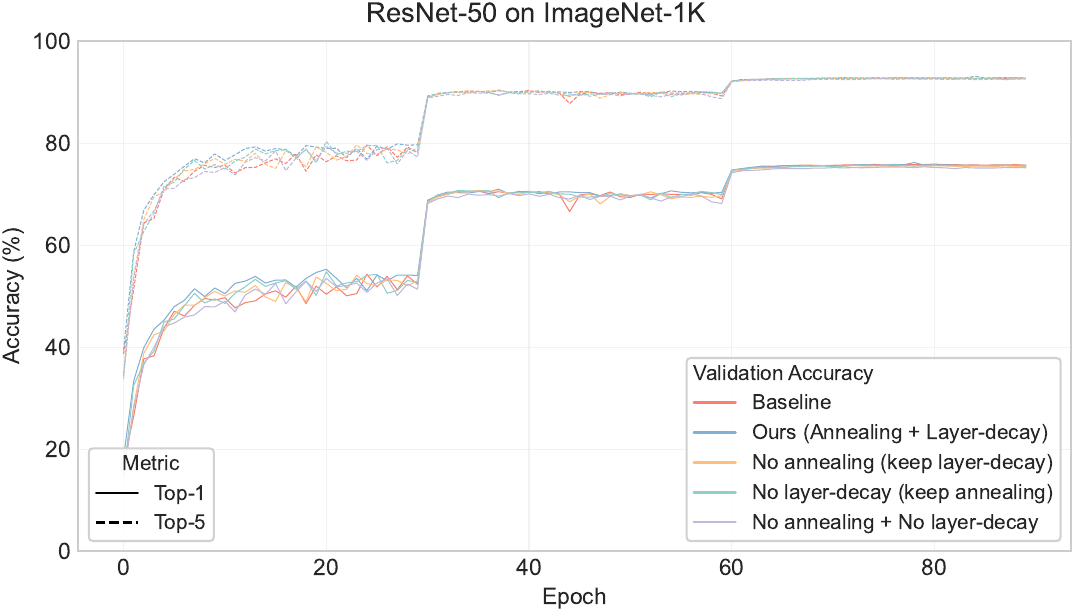}
    \caption{Validation Accuracy}
    \label{fig:ablation_res50_acc}
  \end{subfigure}
  \hfill
  \begin{subfigure}[b]{0.98\textwidth}
    \centering
    \includegraphics[width=\linewidth]{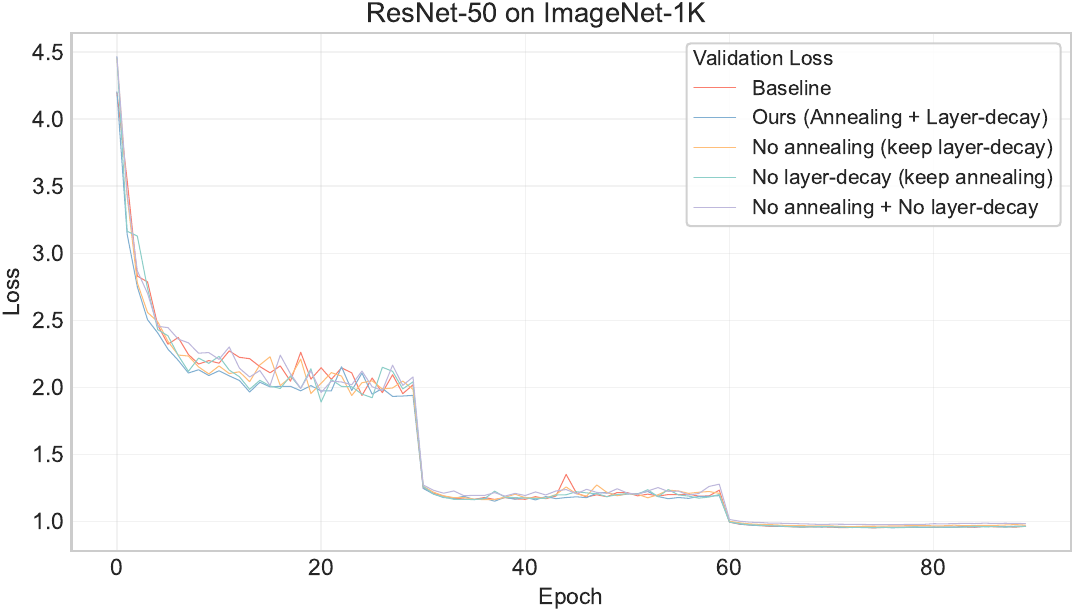}
    \caption{Validation Loss}
    \label{fig:ablation_res50_loss}
  \end{subfigure}
  \caption{\textbf{Ablation on ResNet-50.} The full method (Annealing + Layer Decay) consistently outperforms partial configurations, particularly avoiding the late-stage performance degradation seen in non-annealed settings.}
  \label{fig:ablation_res50}
\end{figure}

\begin{figure}[ht]
  \centering
  \begin{subfigure}[b]{0.98\textwidth}
    \centering
    \includegraphics[width=\linewidth]{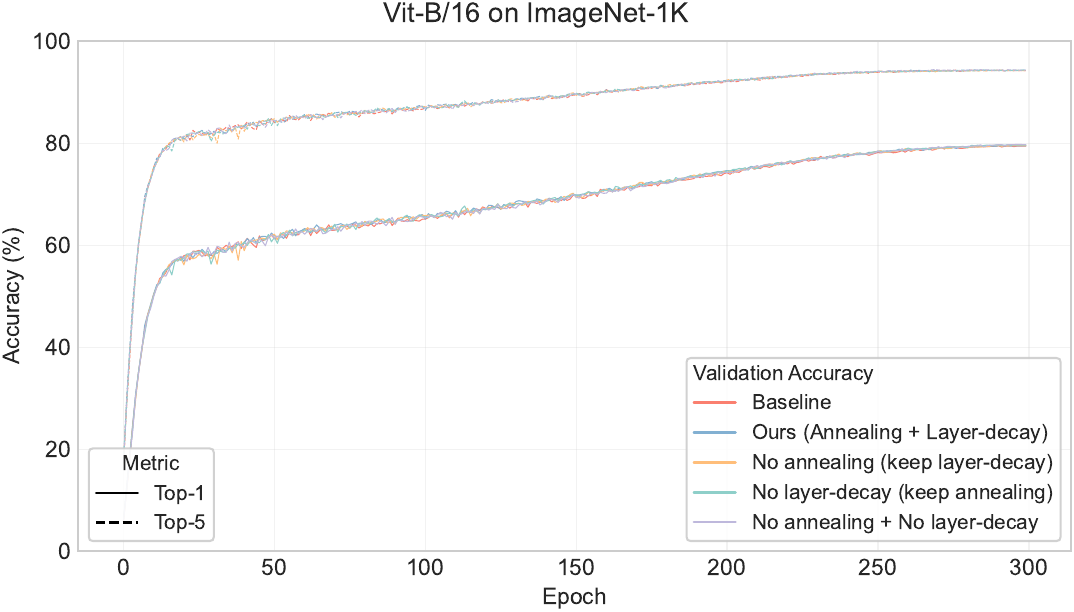}
    \caption{Validation Accuracy}
    \label{fig:ablation_vit_acc}
  \end{subfigure}
  \hfill
  \begin{subfigure}[b]{0.98\textwidth}
    \centering
    \includegraphics[width=\linewidth]{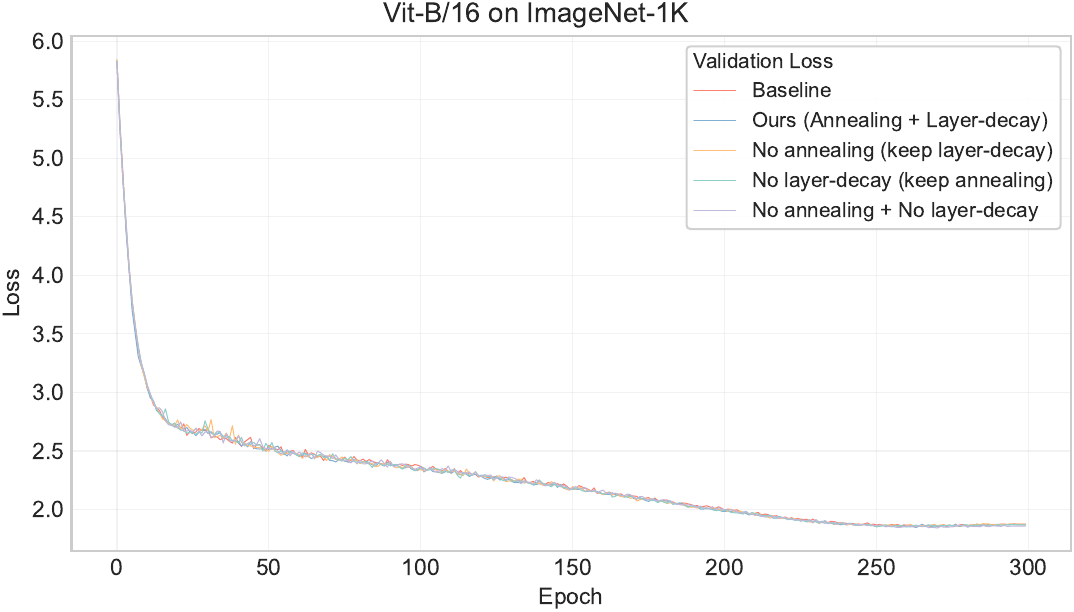}
    \caption{Validation Loss}
    \label{fig:ablation_vit_loss}
  \end{subfigure}
  \caption{\textbf{Ablation on ViT-B/16.} Even for the Transformer architecture, the proposed region-seeding strategy improves early training dynamics and maintains competitive final performance compared to ablated variants.}
  \label{fig:ablation_vit}
\end{figure}

\textbf{1. The Necessity of Annealing.}
Experiments without annealing (Configurations 1 and 2) suffered from degraded final accuracy compared to the experiments with annealing. This observation aligns with our theoretical framework: while minimizing pre-activations (small $|z|$) seeds partitions effectively (Theorem~\ref{thm:distance-intersection}), maintaining this constraint strictly throughout the entire training process conflicts with the classification objective. Annealing resolves this by enforcing geometric seeding early when the partition is being formed, and relaxing it later to allow task-driven feature refinement.

\textbf{2. The Role of Depth-Weighted Decay.}
Comparisons between Uniform Weights and Layer Decay (Configuration 3 vs. 4) demonstrate that penalizing earlier layers more heavily is beneficial. Geometrically, cuts formed by early layers partition the input space globally, whereas deeper layers refine these partitions locally. By focusing the "seeding" pressure on early layers, we encourage a richer set of fundamental hyperplanes that propagate through the network. Conversely, applying high penalties to deep layers (Uniform strategy) can overly restrict the final linear readout's ability to discriminate features, potentially harming performance.

\textbf{Conclusion.}
The combination of \textbf{Annealing} and \textbf{Layer Decay} is essential. The former balances the trade-off between geometric complexity and margin maximization over time, while the latter targets the regularization where it is most topologically effective—at the earlier layers of the hierarchy. This configuration, used in our main results, consistently outperforms the ablated variants.




\end{document}